\documentclass{article}

\usepackage{amsmath,amssymb}
\usepackage[mathscr]{euscript} 
\usepackage{amsthm} 
\usepackage{tikz} 
\usepackage{graphicx} 
\newcommand*{\Scale}[2][4]{\scalebox{#1}{\ensuremath{#2}}} 
\usepackage{framed} 
\usepackage{float} 
\usepackage{multirow}
\usepackage{color}
\usepackage{subfigure}
\usepackage{longtable}
\usepackage{MnSymbol}
\usepackage[numbers,square, comma, sort&compress]{natbib}	 
\usepackage[hidelinks, breaklinks=true]{hyperref} 
\usepackage{xcolor}  \definecolor{shadecolor}{rgb}{.95,.95,.95}  
\usepackage[font=footnotesize]{caption}

\usepackage[ruled,vlined]{algorithm2e}

\newtheorem{myDefinition}{Definition}
\newtheorem{myTheorem}{Theorem}
\newtheorem{myLemma}{Lemma}
\newtheorem{myCorollary}{Corollary}
\newtheorem{myProposition}{Proposition}
\newtheorem{myExample}{Example}
\newtheorem{myRemark}{Remark}

\tikzstyle{every edge}=  [draw]
\tikzstyle{vertex} = [draw,circle,minimum size=1pt]
\tikzstyle{label} = [minimum size=.1pt,font=\scriptsize]
\tikzstyle{title} = [minimum size=.25cm,font=\small]

\newcommand{\bs}[1]{\boldsymbol{#1}}

\newcommand{\fix}[1]{\Scale[.7]{#1}}

\def \R{\mathbb{R}}

\def \P{\mathsf{P}}
\def \T{\mathsf{T}}
\def \c{\mathsf{c}}
\def \spn{{\rm span}}
\def \Ord{\mathscr{O}}
\def \Res{{\hyperref[rDef]{{\rm {\bf Res}}}}}
\def \adj{{\hyperref[adjDef]{\ddagger}}}

\def \r{{\hyperref[rDef]{r}}}
\def \d{{\hyperref[dDef]{d}}}
\def \N{{\hyperref[dDef]{N}}}
\def \Nbreve{{\hyperref[NbreveDef]{\breve{N}}}}
\def \mOf{{\hyperref[mOfDef]{m}}}
\def \nOf{{\hyperref[nOfDef]{n}}}
\def \li{{\hyperref[liDef]{\aleph}}}
\def \L{{\hyperref[LDef]{\ell}}}
\def \eps{{\hyperref[epsDef]{\epsilon}}}
\def \v{{\hyperref[vDef]{{\rm v}}}}
\def \vc{{\hyperref[vcDef]{{\rm v}^\c}}}
\def \muu{{\hyperref[muuDef]{\mu}}}
\newcommand{\Nnum}[1]{{\hyperref[NnumDef]{n_{#1}}}}
\newcommand{\vjNum}[1]{{\hyperref[vjNumDef]{{\rm v}_{{#1}}}}}
\def \wj{{\hyperref[wjDef]{\bs{\gamma}}}}
\def \TT{{\hyperref[TTDef]{T}}}
\def \p{{\hyperref[coherenceFig]{p}}}

\def \nuuT{{\hyperref[nuuTDef]{\nu_{\Theta}}}}
\def \nuuG{{\hyperref[nuuGDef]{\nu_{{\rm G}}}}}
\def \nuuV{{\hyperref[nuuVDef]{\nu_{{\rm V}}}}}
\def \nuuU{{\hyperref[nuuUDef]{\nu_{{\rm U}}}}}

\def \efi{{\hyperref[efiDef]{f_i}}}
\def \efj{{\hyperref[efjDef]{f_j}}}
\def \efihat{{\hyperref[efihatDef]{\hat{f}_i}}}
\def \g{{\hyperref[gDef]{g}}}
\def \gp{{\hyperref[gpDef]{g'}}}
\def \h{{\hyperref[hDef]{h}}}
\def \hh{{\hyperref[hhDef]{h}}}
\def \F{{\hyperref[FDef]{\mathscr{F}}}}
\def \Fp{{\hyperref[FpDef]{\mathscr{F}'}}}
\def \Fdp{{\hyperref[FdpDef]{\mathscr{F}''}}}
\def \Ftilde{{\hyperref[FtildeDef]{\tilde{\mathscr{F}}}}}
\newcommand{\FpNum}[1]{{\hyperref[FpNumDef]{\mathscr{F}'_{#1}}}}
\def \gi{{\hyperref[giDef]{g_{i}}}}
\def \gj{{\hyperref[gjDef]{g_{j}}}}

\def \sstar{{\hyperref[sstarDef]{S^\star}}}
\def \sstaroi{{\hyperref[sstaroiDef]{S^\star_{\bs{\omega}_i}}}}
\def \sstaroihat{{\hyperref[oihatDef]{S^\star_{\bs{\hat{\omega}}_i}}}}
\def \s{{\hyperref[sDef]{S}}}
\def \SS{{\hyperref[SSDef]{\mathscr{S}}}}
\def \VV{{\hyperref[VVDef]{\mathscr{V}}}}

\def \Gr{\hyperref[GrDef]{{\rm Gr}}}
\def \GrND{\hyperref[GrNDDef]{{\rm Gr}_*}}
\def \RND{\hyperref[RNDDef]{\R^{(d-r) \times r}_*}}

\def \x{{\hyperref[xDef]{\bs{{\rm x}}}}}
\def \xoi{{\hyperref[xoiDef]{\bs{{\rm x}}_{\bs{\omega}_i}}}}
\def \xhat{{\hyperref[xhatiDef]{\bs{\hat{{\rm x}}}}}}
\def \xoihat{{\hyperref[xhatiDef]{\bs{\hat{{\rm x}}}_{\bs{\hat{\omega}}_i}}}}
\def \vv{{\hyperref[vvDef]{\bs{{\rm v}}}}}
\def \wwj{{\hyperref[vjNumDef]{\bs{{\rm v}}_{i2}}}}
\def \wwjStar{{\hyperref[vjNumDef]{\bs{{\rm v}}^\star_{i2}}}}
\def \vviStar{{\hyperref[vviDef]{\bs{{\rm v}}^\star_i}}}
\def \ti{{\hyperref[tiDef]{\bs{\theta}_i}}}
\def \tStar{{\hyperref[tStarDef]{\bs{\theta}^\star}}}
\def \tiStar{{\hyperref[tiStarDef]{\bs{\theta}^\star_i}}}
\def \tjStar{{\hyperref[tjStarDef]{\bs{\theta}^\star_{j}}}}
\def \w{{\hyperref[wDef]{\bs{{\rm w}}}}}
\def \vone{{\hyperref[voneDef]{\bs{\gamma}}}}
\def \vtwo{{\hyperref[vtwoDef]{\bs{\gamma}'}}}
\def \ej{{\hyperref[ejDef]{\bs{{\rm e}}_j}}}
\def \aoj{{\hyperref[aojDef]{\bs{{\rm a}}_{\bs{\omega}_i}}}}
\def \aj{{\hyperref[ajDef]{\bs{{\rm a}}_i}}}

\def \I{{\hyperref[IDef]{\bs{{\rm I}}}}}
\def \one{{\hyperref[oneDef]{\bs{{\rm 1}}}}}
\def \X{{\hyperref[XDef]{\bs{{\rm X}}}}}
\def \XO{{\hyperref[XODef]{\bs{{\rm X}}_{\bs{\Omega}}}}}
\def \Xtilde{{\hyperref[XtildeDef]{\bs{\tilde{{\rm X}}}}}}
\def \Xhat{{\hyperref[XhatDef]{\bs{\hat{{\rm X}}}}}}
\def \U{{\hyperref[UDef]{\bs{{\rm U}}}}}
\def \V{{\hyperref[columnEchelonEq]{\bs{{\rm V}}}}}
\def \Ustar{{\hyperref[UstarDef]{\bs{{\rm U}}^\star}}}
\def \Ustaroi{{\hyperref[cdotoDef]{\bs{{\rm U}}^\star_{\bs{\omega}_i}}}}
\def \UstarDeltai{{\hyperref[deltaiDef]{\bs{{\rm U}}^\star_{\bs{\medtriangleup}_i}}}}
\def \UstarNabli{{\hyperref[nabliDef]{\bs{{\rm U}}^\star_{\medtriangledown_i}}}}
\def \Vstar{{\hyperref[columnEchelonEq]{\bs{{\rm V}}}^\star}}
\def \Tstar{{\hyperref[TstarDef]{\bs{\Theta}^\star}}}
\def \Pii{{\hyperref[PiiDef]{\bs{\Pi}}}}
\newcommand{\Vnum}[1]{{\hyperref[VnumDef]{\bs{{\rm V}}_{#1}}}}
\newcommand{\Vcnum}[1]{{\hyperref[VnumDef]{\bs{{\rm V}}^\c_{#1}}}}

\def \Gammaj{{\hyperref[GammajDef]{\bs{\Gamma}_j}}}
\def \Ustarups{{\hyperref[upsDef]{\bs{{\rm U}}^\star_{\bs{\upsilon}}}}}
\def \Pstar{{\hyperref[PstarDef]{\bs{{\rm P}}^\star}}}
\def \A{{\hyperref[ADef]{\bs{{\rm A}}}}}

\def \o{{\hyperref[oDef]{\bs{\omega}}}}
\def \oi{{\hyperref[oiDef]{\bs{\omega}_i}}}
\def \oj{{\hyperref[ojDef]{\bs{\omega}_i}}}
\def \oihat{{\hyperref[oihatDef]{\bs{\hat{\omega}}_i}}}
\def \deltai{{\hyperref[deltaiDef]{\bs{\medtriangleup}_i}}}
\def \deltaj{{\hyperref[deltaiDef]{\bs{\medtriangleup}_{j}}}}
\def \nabli{{\hyperref[nabliDef]{\bs{\medtriangledown}_i}}}
\def \O{{\hyperref[ODef]{\bs{\Omega}}}}
\def \Op{{\hyperref[OpDef]{\bs{\Omega}'}}}
\def \Odp{{\hyperref[FpDef]{\bs{\Omega}''}}}
\def \Ot{{\hyperref[OtDef]{\bs{\hat{\Omega}}_\tau}}}
\def \Otp{{\hyperref[OtpDef]{\bs{\Omega}'_\tau}}}
\def \Otilde{{\hyperref[OtildeDef]{\bs{\tilde{\Omega}}}}}
\def \Ohat{{\hyperref[OhatDef]{\bs{\hat{\Omega}}}}}
\newcommand{\ONum}[1]{{\hyperref[OtDef]{\bs{\hat{\Omega}}_{#1}}}}
\newcommand{\OpNum}[1]{{\hyperref[OtpDef]{\bs{\Omega}'_{#1}}}}
\def \ups{{\hyperref[upsDef]{\bs{\upsilon}}}}
\def \Obreve{{\hyperref[ObreveDef]{\bs{\breve{\Omega}}}}}

\def \i{{\hyperref[iDef]{i}}}
\def \j{{\hyperref[jDef]{j}}}
\def \ttau{{\hyperref[OtDef]{\tau}}}
\def \t{{\hyperref[tDef]{t}}}
\def \finalT{{\hyperref[finalTDef]{T}}}

\def \E{{\hyperref[EDef]{\mathcal{E}}}}
\def \En{{\hyperref[EnDef]{\mathcal{E}_n}}}

\def \ae{{\hyperref[aeDef]{{\rm a.e.}}}}
\def \LRMC{{\hyperref[LRMCDef]{LRMC}}}
\def \fit{{\hyperref[fitsDef]{fit}}}
\def \fits{{\hyperref[fitsDef]{fits}}}
\def \degenerate{{\hyperref[degenerateDef]{degenerate}}}
\def \nondegenerate{{\hyperref[degenerateDef]{non-degenerate}}}
\def \mad{{\hyperref[madDef]{minimally algebraically dependent}}}
\def \finitelyCompletable{{\hyperref[finitelyCompletableDef]{finitely completable}}}
\def \finiteCompletability{{\hyperref[finitelyCompletableDef]{finite completablility}}}
\def \IHTSVD{{\hyperref[IHTSVDDef]{IHTSVD}}}
\def \Aone{{\hyperref[AoneDef]{\textbf{A1}}}}
\def \LRMCThm{{\hyperref[LRMCThm]{Theorem \ref{LRMCThm}}}}
\def \uniquenessThm{{\hyperref[uniquenessThm]{Theorem \ref{uniquenessThm}}}}
\def \probabilityThm{{\hyperref[probabilityThm]{Theorem \ref{probabilityThm}}}}
\def \dimensionLem{{\hyperref[dimensionLem]{Lemma \ref{dimensionLem}}}}

\def \independenceLem{{\hyperref[independenceLem]{Lemma \ref{independenceLem}}}}
\def \LLem{{\hyperref[LLem]{Lemma \ref{LLem}}}}
\def \oneFixedVarLem{{\hyperref[oneFixedVarLem]{Lemma \ref{oneFixedVarLem}}}}
\def \allFixedVarsLem{{\hyperref[allFixedVarsLem]{Lemma \ref{allFixedVarsLem}}}}
\def \basisLem{{\hyperref[basisLem]{Lemma \ref{basisLem}}}}

\def \LRMCLem{{\hyperref[LRMCLem]{Lemma \ref{LRMCLem}}}}
\def \identifiabilityLem{{\hyperref[identifiabilityLem]{Lemma \ref{identifiabilityLem}}}}
\def \rankOneProp{{\hyperref[rankOneProp]{Proposition \ref{rankOneProp}}}}
\def \LRMCCor{{\hyperref[LRMCCor]{Corollary \ref{LRMCCor}}}}
\def \algorithmCor{{\hyperref[algorithmCor]{Corollary \ref{algorithmCor}}}}

\def \modelSec{{\hyperref[modelSec]{Section \ref{modelSec}}}}
\def \implicationsSec{{\hyperref[implicationsSec]{Section \ref{implicationsSec}}}}
\def \proofSec{{\hyperref[proofSec]{Section \ref{proofSec}}}}
\def \uniquenessSec{{\hyperref[uniquenessSec]{Section \ref{uniquenessSec}}}}
\def \additionalSec{{\hyperref[additionalSec]{Section \ref{additionalSec}}}}

\def \samplingEg{{\hyperref[samplingEg]{Example \ref{samplingEg}}}}

\def \necessaryNEg{{\hyperref[necessaryNEg]{Example \ref{necessaryNEg}}}}
\def \adaptiveEg{{\hyperref[adaptiveEg]{Example \ref{adaptiveEg}}}}
\def \entriesRmk{{\hyperref[entriesRmk]{Remark \ref{entriesRmk}}}}
\def \finiteCond{{\hyperref[finiteCond]{{\rm (i)}}}}
\def \identifiabilityCond{{\hyperref[identifiabilityCond]{{\rm (ii)}}}}
\def \measureFig{{\hyperref[measureFig]{Figure \ref{measureFig}}}}
\def \coherenceFig{{\hyperref[coherenceFig]{Figure \ref{coherenceFig}}}}
\def \numItersFig{{\hyperref[numItersFig]{Figure \ref{numItersFig}}}}
\def \regimesFig{{\hyperref[regimesFig]{Figure \ref{regimesFig}}}}
\def \samplesFig{{\hyperref[samplesFig]{Figure \ref{samplesFig}}}}
\def \numItersFig{{\hyperref[numItersFig]{Figure \ref{numItersFig}}}}
\def \discriminantFig{{\hyperref[discriminantFig]{Figure \ref{discriminantFig}}}}
\def \linesFig{{\hyperref[linesFig]{Figure \ref{linesFig}}}}
\def \itersMuFig{{\hyperref[itersMuFig]{Figure \ref{itersMuFig}}}}
\def \identifiabilityCondFig{{\hyperref[identifiabilityCondFig]{Figure \ref{identifiabilityCondFig}}}}
\def \validationAlg{{\hyperref[validationAlg]{Algorithm \ref{validationAlg}}}}

\newcommand*{\titleGP}{\begingroup 
\centering 
\vspace*{\baselineskip} 

\rule{\textwidth}{1.6pt}\vspace*{-\baselineskip}\vspace*{2pt} 
\rule{\textwidth}{0.4pt}\\[.5\baselineskip] 

{\LARGE A Characterization of Deterministic Sampling \\ [0.5\baselineskip] Patterns for Low-Rank Matrix Completion} \\ [0.4\baselineskip] 

\rule{\textwidth}{0.4pt}\vspace*{-\baselineskip}\vspace{3.2pt} 
\rule{\textwidth}{1.6pt}\\[\baselineskip] 

{\scshape 
Daniel L. Pimentel-Alarc\'on, Nigel Boston, Robert D. Nowak}

{\itshape University of Wisconsin - Madison\par}
\endgroup}

\begin{document}
\titleGP

\begin{abstract} 
Low-rank matrix completion (\LRMC) problems arise in a wide variety of applications.
Previous theory mainly provides conditions for completion under missing-at-random samplings.
This paper studies deterministic conditions for completion.
An incomplete $\d \times \N$ matrix is {\em finitely rank-$\r$  completable} if there are at most finitely many rank-$\r$ matrices that agree with all its observed entries.
\hyperref[finitelyCompletableDef]{Finite completability} is the tipping point in \LRMC, as a few additional samples of a \finitelyCompletable\ matrix guarantee its {\em unique} completability.
The main contribution of this paper is a deterministic sampling condition for \finiteCompletability.  We use this to also derive deterministic sampling conditions for unique completability that can be efficiently verified.  We also show that under uniform random sampling schemes, these conditions are satisfied with high probability if $\Ord(\max\{\r,\log\d\})$ entries per column are observed.  These findings have several implications on \LRMC\ regarding lower bounds, sample and computational complexity, the role of coherence, adaptive settings and the validation of any completion algorithm.  We complement our theoretical results with experiments that support our findings and motivate future analysis of uncharted sampling regimes.
\end{abstract}

\section{Introduction}
\label{introSec}
\phantomsection\label{LRMCDef}Low-rank matrix completion (\LRMC) has attracted a lot of attention in recent years because of its broad range of applications, e.g., recommender systems and collaborative filtering \cite{collaborativeFiltering} and image processing \cite{weinberger}.

The problem entails exactly recovering all the entries in a \phantomsection\label{dDef}$\d \times \N$ rank-\phantomsection\label{rDef}$\r$ matrix, given only a subset of its entries.  \LRMC\ is usually studied under a missing-at-random and bounded-coherence model.  Under this model, necessary and sufficient conditions for perfect recovery are known \cite{candes-recht, candes-tao, recht, gross, chen, coherentLRMC}.
Other approaches require additional coherence and spectral gap conditions \cite{bhojanapalli}, use rigidity theory \cite{rigidity}, algebraic geometry and matroid theory \cite{kiraly} to derive necessary and sufficient conditions for completion of deterministic samplings, but a characterization of completable sampling patterns remains an important open question.

We say an incomplete matrix is \phantomsection\label{finitelyCompletableDef}{\em finitely rank-$\r$ completable} if there exist at most finitely many rank-$\r$ matrices that agree with all its observed entries.  There exist sampling/observation patterns that guarantee \finiteCompletability, but if just a single one of the observed entries is instead missing, then there are {\em infinitely} many completions.  Conversely, adding a few observations to such a pattern guarantees {\em unique} completability.  Thus, \finiteCompletability\ is the tipping point in \LRMC.  

Whether a matrix is \finitelyCompletable\ depends on which entries are observed. Yet no characterization of the sets of observed entries that allow or prevent \finiteCompletability\ is known.

The main result of this paper is a sampling condition for \finiteCompletability, that is, a condition on the observed entries of a matrix to guarantee that it can be completed in at most finitely many ways.  In addition, we provide deterministic sampling conditions for unique completability that can be efficiently verified.  Finally, we show that uniform random samplings with $\Ord(\max\{\r,\log\d\})$ entries per column satisfy these conditions with high probability.

Our results have implications on \LRMC\ regarding lower bounds, sample and computational complexity, the role of coherence, adaptive settings and validation conditions to verify the output of any completion algorithm.  We complement our theoretical results with experiments that support our findings and motivate future analysis of uncharted sampling regimes.

\subsection*{Organization of the Paper}
In \modelSec\ we formally state the problem and our main results.  In \implicationsSec\ we discuss their implications in the context of previous work, and present our experiments.  We present the proof of our main theorem in \proofSec, and we leave the proofs of our other statements to Sections \ref{uniquenessSec} and \ref{additionalSec}.

\section{Model and Main Results}
\label{modelSec}
Let \phantomsection\label{XODef}$\XO$ denote the incomplete version of a $\d \times \N$, rank-$\r$ data matrix \phantomsection\label{XDef}$\X$, observed only in the nonzero locations of \phantomsection\label{ODef}$\O$, a $\d \times \N$ matrix with binary entries.  The goal of \LRMC\ is to recover $\X$ from $\XO$.

This problem is tantamount to identifying the $\r$-dimensional subspace \phantomsection\label{sstarDef}$\sstar$ spanned by the columns in $\X$, and this is how we will approach it.  
First observe that since $\X$ is rank-$\r$, a column with fewer than $\r$ samples cannot be completed.  A column with exactly $\r$ observations can be uniquely completed once $\sstar$ is known, but it provides no information to identify $\sstar$.  We will thus assume that:

\begin{shaded}
\begin{itemize}
\phantomsection\label{AoneDef}
\item[\Aone]
Every column of $\X$ is observed on exactly $\r+1$ entries.
\end{itemize}
\end{shaded}

The key insight of the paper is that observing $\r+1$ entries in a column of $\X$
places one constraint on what $\sstar$ may be.  For example, if we observe $\r+1$ entries of a particular
column, then not all $\r$-dimensional subspaces will be consistent with the entries.  If we observe 
more columns with $\r+1$ entries, then even fewer subspaces will be consistent with them.  In effect, each column with $\r+1$ observations places one constraint that an $\r$-dimensional subspace must satisfy in order to be consistent with the observations.  The observed entries in different columns may or may not produce redundant constraints.  As we will see, the pattern of observed entries determines whether or not the constraints are redundant, thus indicating the number of subspaces that satisfy them.
The main result of this paper is a simple condition on the pattern of observed entries that guarantees that only a finite number of subspaces satisfies all the constraints.  This in turn provides a simple condition for exact matrix completion.

\begin{myRemark}
\label{entriesRmk}
We point out that any observation, in addition to the $\r+1$ per column that we assume, cannot increase the number of rank-$\r$ matrices that agree with the observations.  So in general, if some columns of $\X$ are observed on more than $\r+1$ entries, all we need is that the observed entries include a pattern with exactly $\r+1$ observations per column satisfying our sampling conditions.

Also notice that completing $\X$ is the same as completing $\X{}^{\fix{\T}}$, so a row with fewer than $\r$ observations cannot be completed.  While we do not assume that each row is observed on at least $\r$ entries, our sampling conditions guarantee that this is the case.
\end{myRemark}

Let \phantomsection\label{GrDef}$\Gr(\r,\R^{\fix{\d}})$ denote the Grassmannian manifold of $\r$-dimensional subspaces in $\R^{\fix{\d}}$.  Observe that each $\d \times \N$ rank-$\r$ matrix $\X$ can be uniquely represented in terms of a subspace $\sstar \in \Gr(\r,\R^{\fix{\d}})$ (spanning the columns of $\X$) and an $\r\times \N$ coefficient matrix \phantomsection\label{TstarDef}$\Tstar$.  See \measureFig\ to build some intuition.  Let \phantomsection\label{nuuGDef}$\nuuG$ denote the uniform measure on $\Gr(\r,\R^{\fix{\d}})$, and \phantomsection\label{nuuTDef}$\nuuT$ the Lebesgue measure on $\R^{\fix{\r \times \N}}$.  Our statements hold for \phantomsection\label{aeDef}{\em almost every} (\ae) $\X$ with respect to the product measure $\nuuG \times \nuuT$.

\begin{figure}
\centering
\includegraphics[width=3.5cm]{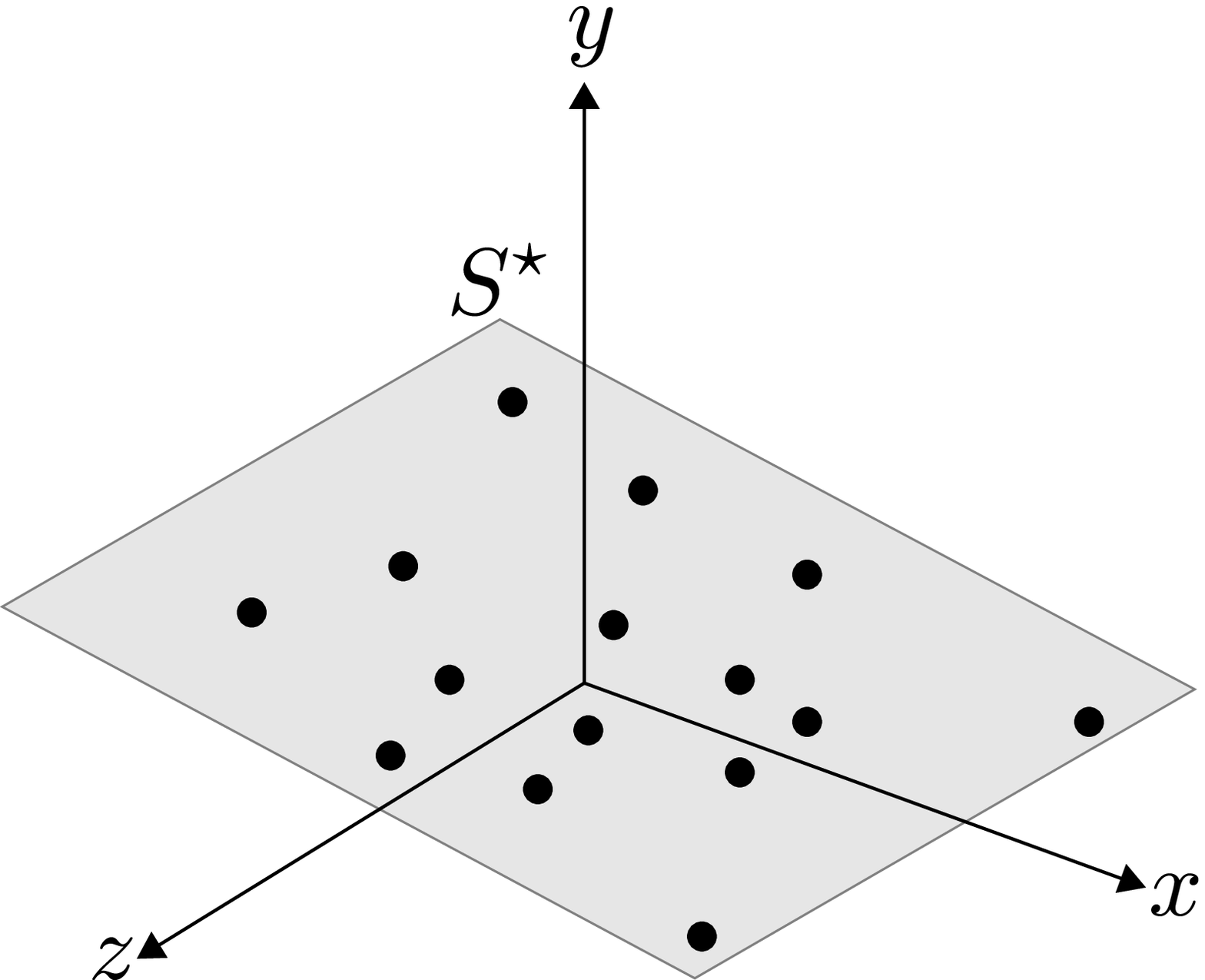} \hspace{.5cm}
\includegraphics[width=3.5cm]{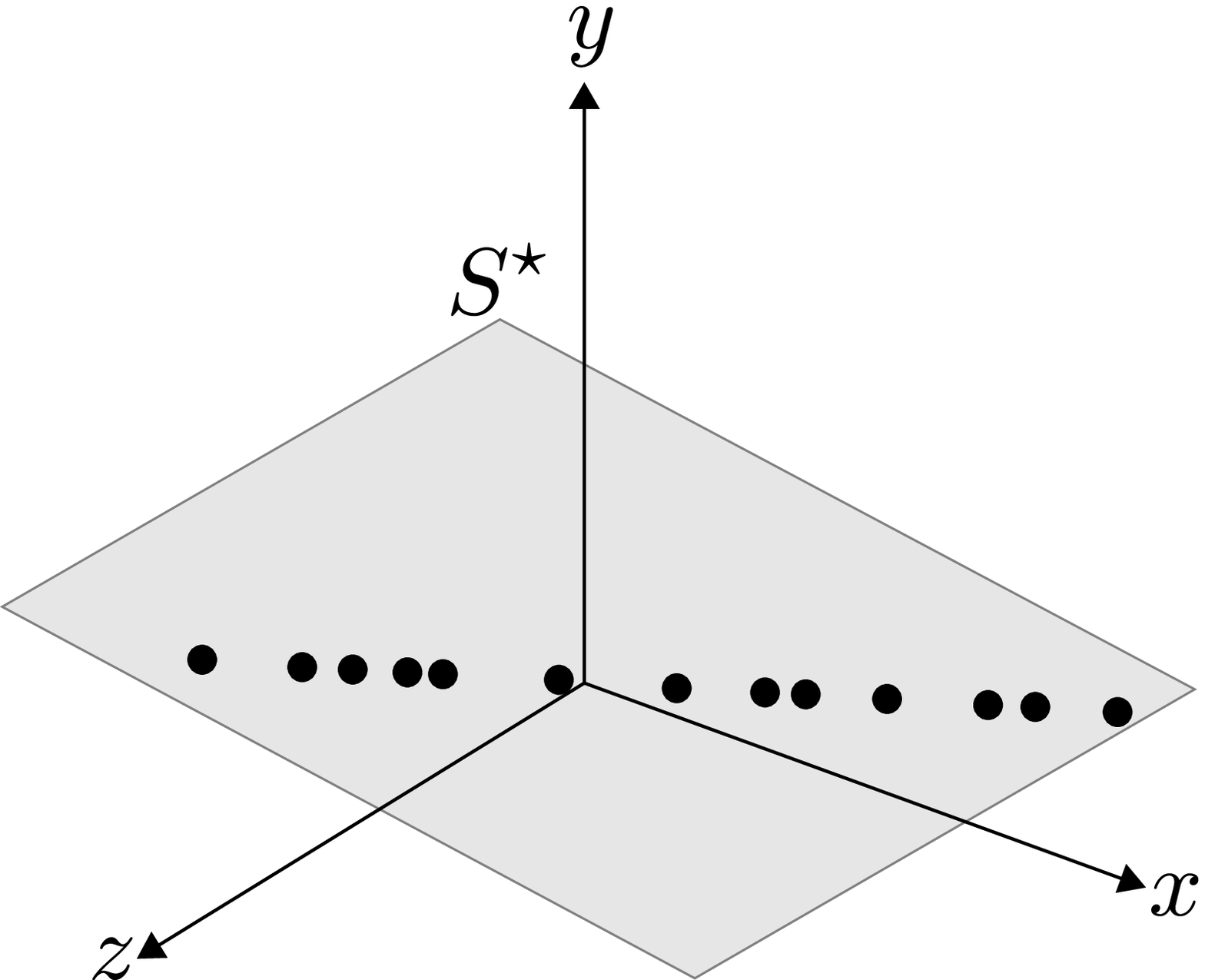}
\caption{Each column in a rank-$\r$ matrix $\X$ corresponds to a point in an $\r$-dimensional subspace $\sstar$.  In these figures, $\sstar$ is a $2$-dimensional subspace (plane) in general position.  In the {\bf left}, the columns of $\X$ are in general position inside $\sstar$, that is, drawn independently according to an absolutely continuous distribution with respect to the Lebesgue measure on $\sstar$, for example, according to a gaussian distribution on $\sstar$.  In this case, the probability of observing a sample as in the {\bf right}, where all columns lie in a line inside $\sstar$, is zero.  Our results hold for every rank-$\r$ matrix, except for a set of measure zero of pathological cases as in the right.}
\label{measureFig}
\end{figure}

The paper's main result is the following theorem, which gives a deterministic sampling condition to guarantee that at most a finite number of $\r$-dimensional subspaces are consistent with $\XO$.

Given a matrix, let $\phantomsection\label{nOfDef}\nOf(\bs{\cdot})$ denote its number of columns and \phantomsection\label{mOfDef}$\mOf(\bs{\cdot})$ the number of its {\em nonzero} rows.

\begin{framed}
\begin{myTheorem}
\label{LRMCThm}
Let $\O$ be given, and suppose \Aone\ holds.  For \hyperref[aeDef]{almost every} $\X$, there exist at most finitely many rank-$\r$ completions of $\XO$ if and only if there exists a matrix \phantomsection\label{OtildeDef}$\Otilde$ formed with $\r(\d-\r)$ columns of $\O$, such that
\begin{itemize}
\item[\finiteCond]
\phantomsection\label{finiteCond}
Every matrix \phantomsection\label{OpDef}$\Op$ formed with a subset of the columns in $\Otilde$ satisfies
\begin{align}
\label{LRMCEq}
\mOf(\Op) \ \geq \ \nOf(\Op)/\r + \r.
\end{align}
\end{itemize}
\end{myTheorem}
\end{framed}

The proof of \LRMCThm\ is given in \proofSec.  In words, condition \finiteCond\ asks that every subset of $\nOf$ columns of $\Otilde$ has at least $\nOf/\r+\r$ nonzero rows.

\begin{myExample}
\label{samplingEg}
Suppose $\O$ is given by:
\begin{align*}
\O \ = \ \left[ \begin{matrix} \\ \\ \\ \\ \end{matrix} \right.
\underbrace{
\begin{array}{c|c|c|c}
\hspace{.3cm} \Scale[1.5]{\one} \hspace{.3cm} & \hspace{.3cm} \Scale[1.5]{\one} \hspace{.3cm} & \cdots & \hspace{.3cm} \Scale[1.5]{\one} \hspace{.3cm} \\ \hline
& & & \\
\Scale[1.5]{\I} & \Scale[1.5]{\I} & \cdots & \Scale[1.5]{\I} \\
& & & \\
\end{array}}_{\fix{(\r+1)(\d-\r)}}
\left]\begin{matrix}
\left. \begin{matrix} \\ \end{matrix} \right\} \r \hspace{.7cm} \\
\left. \begin{matrix} \\ \\ \\ \end{matrix} \right\} \d-\r,
\end{matrix}
\right.
\end{align*}
such that $\O$ has exactly $\r+1$ nonzero entries per column.  This way, each column of $\O$ encodes exactly one constraint that candidate subspaces must satisfy in order to be consistent with the observed data.  In this case we can simply take $\Otilde$ to be the matrix formed with the first $\r(\d-\r)$ columns of $\O$.  One may verify that $\Otilde$ satisfies \finiteCond.  Hence $\O$ satisfies the conditions of \LRMCThm.
\end{myExample}

\subsection*{Unique Completability}

\LRMCThm\ is easily extended to a condition on $\O$ that is sufficient to guarantee that one and only one subspace is consistent with $\XO$, which in turn suffices for exact matrix completion.

\begin{framed}
\begin{myTheorem}
\label{uniquenessThm}
Let $\O$ be given, and suppose \Aone\ holds.  Then \hyperref[aeDef]{almost every} $\X$ can be uniquely recovered from $\XO$ if $\O$ contains two disjoint submatrices: $\Otilde$ of size $\d \times \r(\d-\r)$ and \phantomsection\label{OhatDef}$\Ohat$ of size $\d \times (\d-\r)$, such that $\Otilde$ satisfies \finiteCond\ and
\begin{itemize}
\item[\identifiabilityCond]
\phantomsection\label{identifiabilityCond}
Every matrix $\Op$ formed with a subset of the columns in $\Ohat$ satisfies
\begin{align}
\label{identifiabilityEq}
\mOf(\Op) \ \geq \ \nOf(\Op) + \r.
\end{align}
\end{itemize}
\end{myTheorem}
\end{framed}

The proof of \uniquenessThm\ is given in \uniquenessSec.  In words, condition \identifiabilityCond\ asks that every subset of $\nOf$ columns of $\Ohat$ has at least $\nOf+\r$ nonzero rows.  Notice that \eqref{LRMCEq} is a weaker condition than \eqref{identifiabilityEq}, but \eqref{LRMCEq} is required to hold for all the subsets of $\r(\d-\r)$ columns, while \eqref{identifiabilityEq} is required to hold only for all the subsets of $\d-\r$ columns.

\begin{myExample}
\label{uniquenessEg}
Consider $\O$ as in \samplingEg. Take $\Otilde$ to be the matrix formed with the first $\r(\d-\r)$ columns of $\O$ and $\Ohat$ to be the matrix formed with the last $\d-\r$ columns of $\O$. One may verify that $\Otilde$ satisfies \finiteCond\ and that $\Ohat$ satisfies \identifiabilityCond.  Hence $\O$ satisfies the conditions of \uniquenessThm.
\end{myExample}

\LRMCThm\ implies that $\r(\d-\r)$ columns with $\r+1$ entries are necessary for \finiteCompletability\ (hence also for unique completability).  There are cases when $\r(\d-\r)$ columns are also sufficient for unique completability, e.g., if $\r=1$, where \finiteCompletability\ is equivalent to unique completability (see \rankOneProp).

In general, though, unique completability requires more columns than \finiteCompletability\ (see \necessaryNEg).  \uniquenessThm\ gives deterministic sufficient sampling conditions for unique completability that only require $(\r+1)(\d-\r)$ columns.  This shows that with just a few more observations, unique completability follows from \finiteCompletability.

We point out that when the conditions of \uniquenessThm\ are met, $\sstar$ can be uniquely identified as
\begin{align*}
\sstar = \spn\left[ \begin{matrix}\I \\ \V \end{matrix} \right],
\end{align*}
where $\V$ is the unique solution to the polynomial system $\F(\V)=\bs{0}$, with $\F$ as defined in \proofSec.

Once $\sstar$ is known, $\X$ can be perfectly recovered observing only $\r$ entries per column.  To see this, let \phantomsection\label{UstarDef}$\Ustar$ be a basis of $\sstar$, and let \phantomsection\label{upsDef}$\ups$ be a subset of $\{1,\dots,\d\}$ with exactly $\r$ elements.  We will use the subscript $\ups$ to denote restriction to the rows in $\ups$.  Since the coefficients of column \phantomsection\label{xDef}$\x$ in the basis $\Ustar$ are given by \phantomsection\label{tStarDef}$\tStar= (\Ustarups{}^{\fix{\T}} \Ustarups)^{-1} \Ustarups{}^{\fix{\T}} \x_{\fix{\ups}}$, we can recover the entire column as $\x=\Ustar \tStar$.

\subsection*{About Conditions \finiteCond\ and \identifiabilityCond}
In general, verifying condition \finiteCond\ in Theorems \ref{LRMCThm} and \ref{uniquenessThm} may be computationally prohibitive, especially for large $\d$.  On the other hand, one can easily and efficiently verify whether \identifiabilityCond\ is satisfied by checking the dimension of the null-space of a sparse matrix (\validationAlg).  Fortunately, there is a tight relation between conditions \finiteCond\ and \identifiabilityCond, summarized in the following lemma.  The proofs of the statements in this section are given in \additionalSec.

\begin{myLemma}
\label{LRMCLem}
Let $\Otilde$ be a $\d \times \r(\d-\r)$ matrix formed with a subset of the columns in $\O$.  Suppose $\Otilde$ can be partitioned into $\r$ matrices \phantomsection\label{OtDef}$\{\Ot\}_{\fix{\ttau=1}}^{\fix{\r}}$, each of size $\d \times (\d-\r)$, such that \identifiabilityCond\ holds for every $\Ot$.  Then $\Otilde$ satisfies \finiteCond.
\end{myLemma}

As consequence of \LRMCLem\ we obtain an additional sufficient condition for completability that only involves \identifiabilityCond.

\begin{myCorollary}
\label{LRMCCor}
Let $\O$ be given, and suppose \Aone\ holds.  Then \hyperref[aeDef]{almost every} $\X$ can be uniquely recovered from $\XO$ if $\O$ contains $\r+1$ disjoint matrices $\{\Ot\}_{\fix{\ttau=1}}^{\fix{\r+1}}$, each of size $\d \times (\d-\r)$, such that \identifiabilityCond\ holds for every $\Ot$.
\end{myCorollary}

\begin{myExample}
\sloppypar Consider $\O$ as in \samplingEg.  We can partition $\O$ into $[ \ \ONum{1} \ | \ \ONum{2} \ | \ \cdots \ | \ \ONum{r+1} ]$, as depicted in \samplingEg.  One may verify that $\Ot$ satisfies \identifiabilityCond\ for every $\ttau=1,\dots,\r+1$.  Hence $\O$ satisfies the conditions of \LRMCCor.
\end{myExample}

With \LRMCCor\ we show that completable patterns appear with high probability under uniform random sampling schemes with as little as $\Ord(\max\{ \r , \log\d\})$ samples per column.

\begin{myTheorem}
\label{probabilityThm}
Let \phantomsection\label{epsDef}$0<\eps \leq 1$ be given.  Suppose $\r \leq \frac{\fix{\d}{}}{6}$ and that each column of $\X$ is observed in at least \phantomsection\label{LDef}$\L$ entries, distributed uniformly at random and independently across columns, with
\begin{align}
\label{LEq}
\textstyle \L \ \geq \ \max \left\{12\left( \log(\frac{\fix{\d}{}}{\fix{\eps}{}})+1\right), \ 2\r \right\}.
\end{align}
Then with probability at least $1-\eps$, $\XO$ will be \hyperref[finitelyCompletableDef]{finitely rank-$\r$ completable} (if $\N \geq \r(\d-\r)$) and uniquely completable (if $\N \geq (\r+1)(\d-\r)$).
\end{myTheorem}

In many situations, though, sampling is not uniform.  For instance, in vision, occlusion of objects can produce missing data in very non-uniform random patterns.  In cases like this, we can partition $\O$ (e.g., randomly) into matrices $\{\Ot\}_{\fix{\ttau=1}}^{\fix{\r+1}}$, each with $\d-\r$ columns.  We can use \validationAlg\ below to determine whether each $\Ot$ satisfies \identifiabilityCond.  If this is the case, $\O$ is completable by \LRMCCor.  More about this is discussed in \implicationsSec.

To present the algorithm, let us introduce the matrix \phantomsection\label{ADef}$\A$ that will allow us to determine efficiently whether a sampling $\Ohat$ satisfies \identifiabilityCond.  Let \phantomsection\label{ObreveDef}$\Obreve$ be a matrix formed with \phantomsection\label{NbreveDef}$\Nbreve \geq \d-\r$ columns of $\O$, and let \phantomsection\label{ojDef}$\oj$ index the nonzero entries in the $\i^{\rm th}$ column of $\Obreve$.  Let \phantomsection\label{UDef}$\U$ be a $\d \times \r$ matrix drawn according to \phantomsection\label{nuuUDef}$\nuuU$, an absolutely continuous distribution with respect to the Lebesgue measure on $\R^{\fix{\d \times \r}}$, and let $\U_{\fix{\oj}}$ denote the restriction of $\U$ to the nonzero rows in $\oj$.  Let \phantomsection\label{aojDef}$\aoj \in \R^{\fix{\r+1}}$ be a nonzero vector in $\ker \U{}_{\fix{\oj}}^{\fix{\T}}$, and \phantomsection\label{ajDef}$\aj$ be the vector in $\R^{\fix{\d}}$ with the entries of $\aoj$ in the nonzero locations of $\oj$ and zeros elsewhere.  Finally, let $\A$ denote the $\d \times \Nbreve$ matrix with $\{\aj\}_{\fix{\i=1}}^{\fix{\Nbreve}}$ as columns.

\begin{algorithm}[tb]
\caption{Determine whether $\Obreve$ contains a matrix $\Ohat$ satisfying \identifiabilityCond.}
\label{validationAlg}

\textbf{Input:} Matrix $\Obreve$ with $\Nbreve \geq \d-\r$ columns of $\O$.

- Draw $\U \in \R^{\fix{\d \times \r}}$ according to $\nuuU$.

- \For{$\i=1$ \KwTo $\Nbreve$}{

- $\oi$ = indices of the nonzero rows of the $\i^{\rm th}$ column of $\Obreve$.

- $\aoj = $ nonzero vector in $\ker \U{}_{\fix{\oj}}^{\fix{\T}}$.

- $\aj = $ vector in $\R^{\fix{\d}}$ with entries of $\aoj$ in the

\hspace{.85cm} nonzero locations of $\oj$ and zeros elsewhere.

}

- $\A = $ matrix formed with $\{\aj\}_{\fix{\i=1}}^{\fix{\Nbreve}}$ as columns.

- \uIf{$\dim \ker \A{}^{\fix{\T}} = \r$}{

- \textbf{Output:} $\Obreve$ contains a $\d \times (\d-\r)$ matrix $\Ohat$ satisfying \identifiabilityCond.

}

- \Else{

- \textbf{Output:} $\Obreve$ contains no $\d \times (\d-\r)$ matrix $\Ohat$ satisfying \identifiabilityCond.

}

\end{algorithm}

\validationAlg\ will verify whether $\dim \ker \A{}^{\fix{\T}} = \r$, and this will determine whether $\Obreve$ contains a $\d \times (\d-\r)$ matrix $\Ohat$ satisfying \identifiabilityCond.  The key insight behind \validationAlg\ is that $\A$ encodes the information of the projections of \phantomsection\label{sDef}$\s=\spn\{\U\}$ onto the canonical coordinates indicated by $\Obreve$.  Theorem 1 in \cite{identifiability} shows that these projections will uniquely determine $\s$ if and only if $\dim \ker \A{}^{\fix{\T}} = \r$, which will be the case if and only if $\Obreve$ contains a $\d \times (\d-\r)$ matrix $\Ohat$ satisfying \identifiabilityCond.  We thus have the following corollary, which states that with probability $1$, \validationAlg\ will determine whether $\Obreve$ contains a matrix $\Ohat$ satisfying \identifiabilityCond.

\begin{myCorollary}
\label{algorithmCor}
Let $\Obreve$ be a matrix formed with $\Nbreve \geq \d-\r$ columns of $\O$.  Construct $\A$ as in \validationAlg.  Then $\nuuU$-almost surely, $\Obreve$ contains a $\d \times (\d-\r)$ matrix $\Ohat$ satisfying \identifiabilityCond\ if and only if $\dim \ker \A{}^{\fix{\T}} = \r$.
\end{myCorollary}

\validationAlg\ can also be used to design completable samplings.  As will be discussed in \implicationsSec, this can be particularly useful for adaptive settings, where one may choose which entries to observe, yet it is undesirable or impossible to observe full columns or full rows.

\begin{myExample}
\label{adaptiveEg}
One may use \validationAlg\ to verify that each of the $\r+1$ blocks in the sampling matrix $\O$ below satisfies \identifiabilityCond.  This implies that $\O$ satisfies the conditions of \LRMCCor, and can thus be uniquely completed.  In contrast with \samplingEg, this pattern does not sample full columns nor full rows.
\begin{figure}[H]
\centering
\includegraphics[width=8cm]{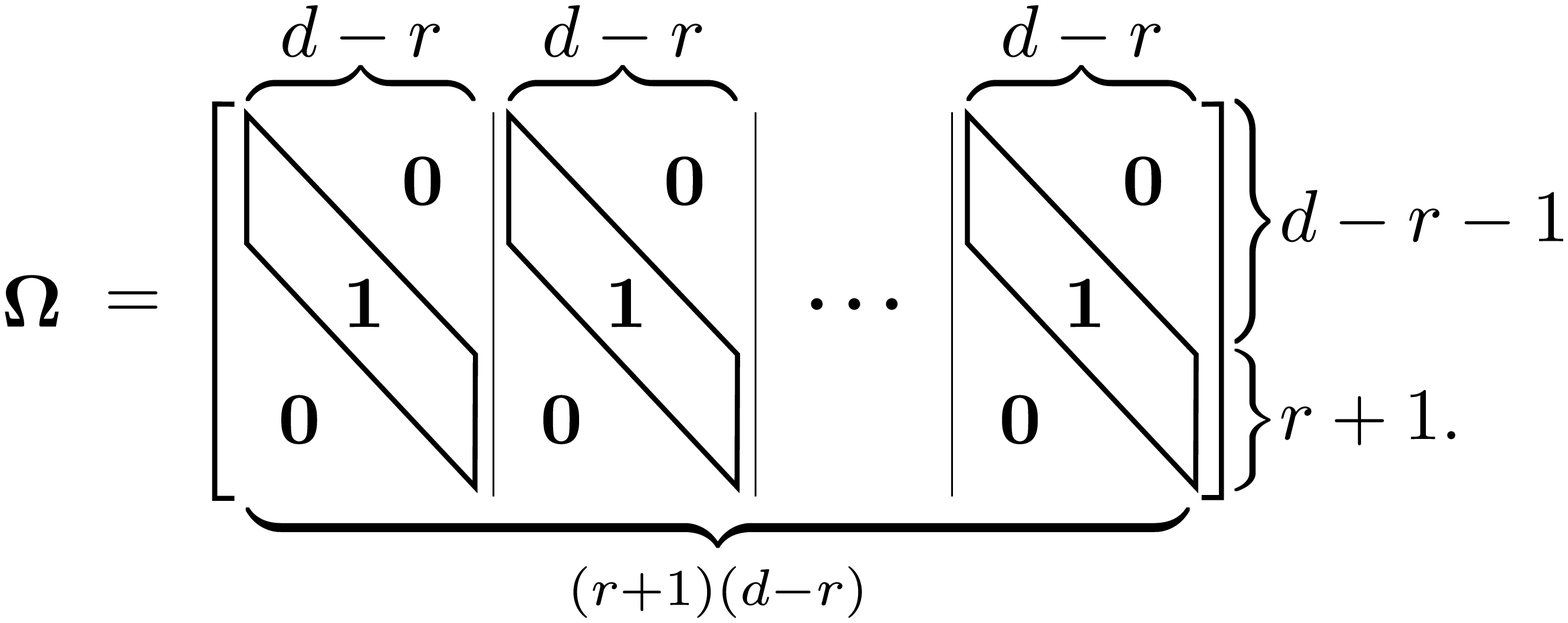}
\label{OmegaEgFig}
\end{figure}
\end{myExample}

\section{Experiments and Implications}
\label{implicationsSec}

In this section we discuss implications of the results stated above and explore how well they predict
performance in a series of simulation experiments.  In all our experiments we use the so-called \phantomsection\label{IHTSVDDef}iterative hard-thresholded SVD (\IHTSVD) algorithm \cite{iterative}.  This algorithm iterates between truncating the SVD of the current estimate to a user-specified rank $r$, and then replacing the values in the observed entries with their original (observed) values. This algorithm is also quite similar to the Singular Value Thresholding algorithm \cite{cai}, OptSpace \cite{keshavan10} and FPCA \cite{fpc}.  In the very low sampling regimes of interest in our studies, we found the \IHTSVD\ algorithm typically performed as well or better than several other completion algorithms (e.g., SVT \cite{cai}, GROUSE \cite{grouse}, alternating minimization \cite{jain} and EM \cite{ssp14}).

\subsection*{Lower bound}
It is easy to see that $\L=\Ord(\max\{ \r , \log\d\})$ uniformly randomly sampled entries per column are necessary to complete an $\Ord(\d) \times \Ord(\d)$ matrix.  This is because a column with fewer than $\r$ observed entries cannot be completed, and if fewer than $\Ord(\log\d)$ uniformly random samples per column are observed, then a row may be completely unobserved with large probability, making it impossible to complete a matrix.  Thus, $\L=\Ord(\max\{ \r , \log\d\})$ is a  lower bound for \LRMC.

It was further shown \cite{candes-tao} that there exist matrices that cannot be completed unless $\L=\Ord(\muu\r\log\d)$ uniformly randomly sampled entries per column are observed, where \phantomsection\label{muuDef}$\muu \in [1,\frac{\fix{\d}{}}{\fix{\r}{}}]$ is the standard coherence parameter defined as
\begin{align*}
\muu \ := \ \textstyle \frac{\fix{\d}}{\fix{\r}{}} \max_{\fix{1 \leq \j \leq \d}} \| \Pstar \ej \|_2^2,
\end{align*}
where \phantomsection\label{PstarDef}$\Pstar$ denotes the projection operator onto $\sstar$, and \phantomsection\label{ejDef}$\ej$ the $\j^{\rm th}$ canonical vector in $\R^{\fix{\d}}$.  Our results imply that this is only the case for a set of matrices with measure zero, and that \ae\ matrix can be uniquely completed with as little as $\L=\Ord(\max\{ \r , \log\d\})$ uniformly randomly sampled samples per column, regardless of $\muu$.

To better understand this, and see that our results do not contradict previous theory, let us revisit the proof of Theorem 1.7 in \cite{candes-tao}.  The proof is based on the construction of block-diagonal matrices with blocks of size $\frac{\fix{\d}{}}{\fix{\r}\fix{\muu}{}}$ and coherence $\leq \muu$ that cannot be recovered with fewer than $\L=\Ord(\muu\r\log\d)$ uniformly random samples per column, e.g.,
\begin{align}
\label{blockDiagonalEq}
\X \ = \ 
\left[ \begin{matrix} \\ \\ \\ \\ \\ \\ \end{matrix} \right.
\overbrace{
\begin{array}{c}
\multirow{2}{*}{\hspace{.1cm}\Scale[1]{\bs{B}_1}\hspace{.1cm}} \\ \\ \hline
\multirow{3}{*}{\Scale[1.5]{\bs{0}}} \\ \\ \\
\end{array}}^{\frac{\fix{\d}{}}{\fix{\r \muu}{}}}
\begin{array}{|c|}
\multicolumn{1}{|c}{\multirow{2}{*}{}} \\ \multicolumn{1}{|c}{} \\  \hline
\multirow{1}{*}{$\ddots$} \\ \hline
\multicolumn{1}{c|}{\multirow{2}{*}{}} \\ \multicolumn{1}{c|}{} \\ 
\end{array}
\underbrace{
\begin{array}{c}
\multirow{3}{*}{\Scale[1.5]{\bs{0}}} \\ \\ \\ \hline
\multirow{2}{*}{\Scale[1]{\bs{B}_{\fix{\r \muu}}}} \\ \\
\end{array}}_{\frac{\fix{\d}{}}{\fix{\r \muu}{}}}
\left. \left. \begin{matrix} \\ \\ \\ \\ \\ \\ \end{matrix} \right] \right\}\d.
\end{align}
This is so because zero valued entries provide no information for the reconstruction process.  It follows that the larger $\muu$, the smaller the blocks will be, and more intensive random sampling would be required to guarantee that entries in the diagonal blocks are observed.  This is why more samples ($\Ord( \r\muu\log\d)$ per column) are required to reconstruct more coherent matrices {\em like this one}, and hence the dependency on $\r\muu$ in the bound of Theorem 1.7 in \cite{candes-tao}.

However, matrices with this block structure have measure zero (with respect to the measure defined above).  Our results show that for \ae\ matrix, an incomplete column contains the same exploitable information regardless of the coherence parameter, and $\Ord(\max\{\r,\log\d\})$ uniform random entries per column are sufficient for completion. This means that while there are some matrices that require $\Ord(\r\muu\log\d)$ uniform random samples per column for reconstruction, \ae\ matrix only requires $\Ord(\max\{\r,\log\d\})$, regardless of $\muu$.

\subsection*{Sample Complexity}
Coherence aside, it is also known that $\N=\d$ columns, and $\L=\Ord(\r \log\d)$ uniform random samples per column are sufficient for completion \cite{candes-recht}.  \probabilityThm\ extends this result, showing that $\N=(\r+1)(\d-\r)$ columns and $\L=\Ord(\max\{ \r , \log\d\})$ uniform random samples per column are sufficient to uniquely complete \ae\ matrix.  This exposes an interesting tradeoff between the required number of columns and observed entries per column for completion, defining new unstudied sampling regimes where completion is now known to be possible (\regimesFig).

\begin{figure}
\centering
\includegraphics[width=8cm]{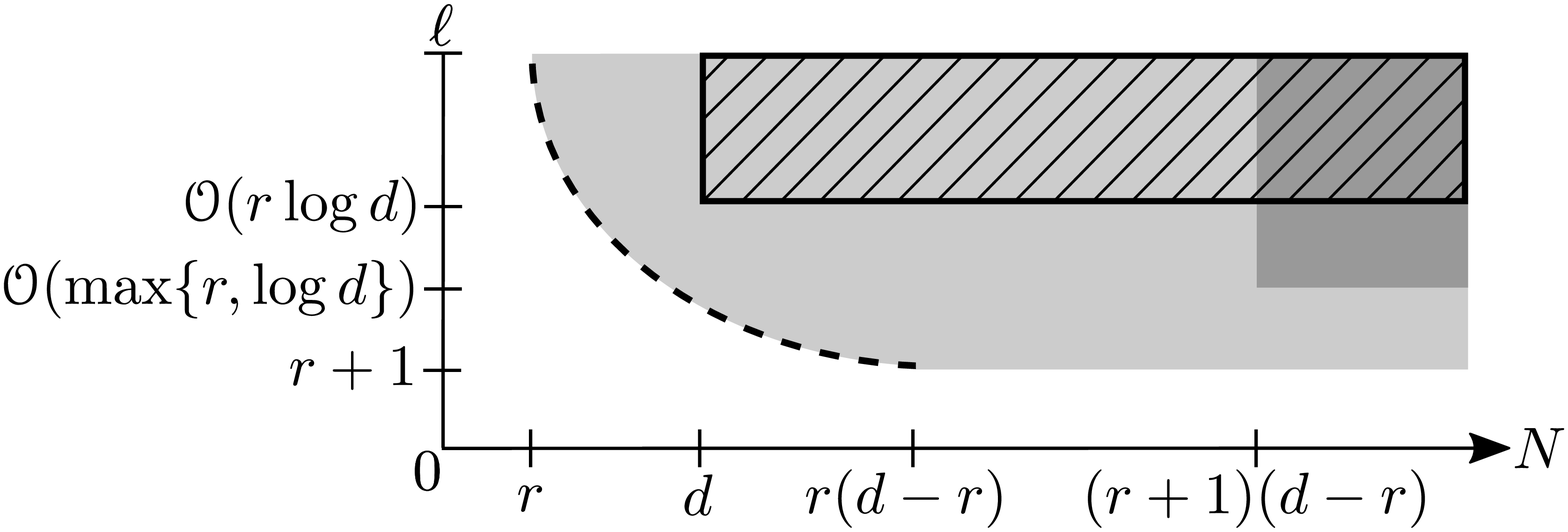}
\caption{Theoretical sampling regimes of \LRMC.   In the white region, where the dashed line is given by $\L = \frac{\fix{\r}(\fix{\d}-\fix{\r})}{\fix{\N}{}}+\r$, it is easy to see that \LRMC\ is impossible by a simple count of the degrees of freedom in a subspace (see \proofSec).  In the light-gray region, \LRMC\ is possible provided the entries are observed in the right places, e.g., satisfying the conditions of \uniquenessThm.  By \probabilityThm, uniform random samplings will satisfy these conditions with high probability as long as $\N \geq (\r+1)(\d-\r)$ and $\L \geq \max \{12 ( \log(\frac{\fix{\d}{}}{\fix{\eps}{}})+1), \ 2\r \}$, hence with high probability \LRMC\ is possible in the dark-grey region.  Previous analyses showed that \LRMC\ is possible from uniform random sampling in the striped region \cite{candes-recht}, but the rest remained unclear until now.}
\label{regimesFig}
\end{figure}

The purpose of our first experiment is to support that $\L=\Ord(\max\{ \r , \log\d\})$ random samples per column are truly sufficient for \LRMC, as opposed to $\Ord(\r \log \d)$.  To this end, we will study the behavior of the \IHTSVD\ algorithm as a function of the ambient dimension $\d$ and the rank $\r$ (see the beginning of \implicationsSec\ for a discussion of this algorithmic choice). 

To obtain low-rank matrices, we first generated a $\d \times \r$ random matrix $\Ustar$ with $\mathscr{N}(0,1)$ i.i.d. entries to use as basis of $\sstar$.  We then generated an $\r \times (\r+1)(\d-\r)$ random matrix $\Tstar$, also with $\mathscr{N}(0,1)$ i.i.d. entries, to use as coefficient vectors, to construct $\X=\Ustar \Tstar$.  Matrices generated this way are known to have low coherence.

Next, for different values of the rank $\r$, we tested whether a matrix could be completed as a function of its ambient dimension $\d$ and the number of uniform random samples per column $\L$. For example, the results of this experiment for $\r=7$ can be seen in \discriminantFig.

\begin{figure}
\centering
\includegraphics[width=7cm]{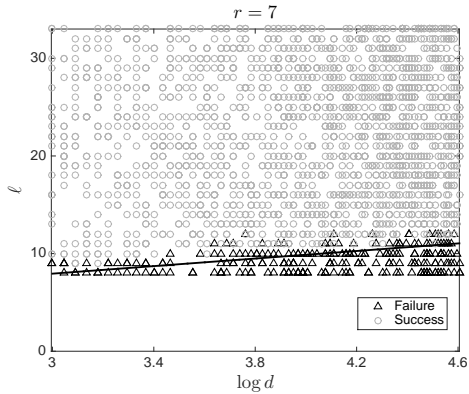}
\caption{Results of the \IHTSVD\ algorithm as a function of the ambient dimension $\d$ and the number of uniform random samples per column $\L$, for rank $\r=7$.  In each of the $2,000$ trials we declared a success if the normalized completion error was below $10^{-12}$ (using normalized Frobenius norm).  The black line represents the linear discriminant between success and failure trials.}
\label{discriminantFig}
\end{figure}

We then computed the linear discriminant between successful and failure trials for each value of $\r$.  If $\L=\Ord(\r\log\d)$ samples were necessary, we would expect the slope between these lines to grow proportionally to $\r$.   However, the results, depicted in \linesFig, show that the slope of these lines remain fairly constant, and the offset grows with $\r$, supporting that $\L=\Ord(\max\{\r,\log\d\})$ samples are sufficient.

\begin{figure}
\centering
\includegraphics[width=5.5cm]{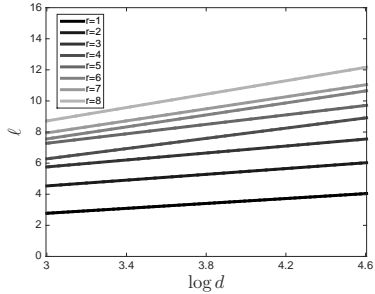}
\caption{Linear discriminants for different values of the rank $\r$, between successful (above line) and unsuccessful (below line) completions for the experiment in \discriminantFig.  That is, for a given $\r$, any pair $(\log\d,\L)$ above the linear discriminant typically succeeds at completion, and below the linear discriminant typically fails.   \probabilityThm\ shows that $\L=\Ord(\max\{\r,\log\d\})$ uniform random observations per column are sufficient for completion.  The slope of these lines remain fairly constant, and the offset grows with $\r$, supporting this result.}
\label{linesFig}
\end{figure}

\subsection*{Computational Complexity}
Our results show that completion is theoretically possible with as little as with $\L \geq \Ord(\max\{\r,\log \d\})$ uniform random samples per column, or even with as little as $\L=\r+1$ (provided they are located in the right places).  Nevertheless, this may involve solving the system of polynomial equations $\F=\bs{0}$ (see \proofSec), which is computationally impractical.  It is thus currently unknown whether there exist practical completion algorithms for these uncharted sampling regimes.

We now present a series of experiments that suggest three things: first, that even in cases where \LRMC\ is theoretically possible, missingness seems to come at a price: the more missing data the more computationally expensive completion seems to be.  This further suggests that there is a minimal sampling regime where, though theoretically possible, \LRMC\ might be computationally prohibitive in practice.  Second, that even though theoretically, whether \ae\ matrix can be completed does not depend on its coherence, in practice, extremely coherent matrices may be computationally more expensive to complete.  Similarly, this suggests that there is a maximal coherence regime where, though theoretically possible, \LRMC\ might be computationally prohibitive in practice.  And third, there seems to be an additional uncharted sampling regime with $\L<\Ord(\muu\r\log\d)$ samples per column where completion is computationally feasible.

To summarize, we have the following sampling regimes, where \LRMC\ is:
\begin{center}
	\begin{tikzpicture}
		\node [label] (a) at (-3.9,0) {$$};
		\node [label] (z) at (4,0) {$$};
		\node [title] (ell) at (4,0) {$\L$};
		
		\node [title] (b) at (-2.6,0) {$\bs{|}$}; \node [title] (b) at (-2.6,-.4) {$\r+1$};
		\node [title] (c) at (0,0) {$\bs{|}$}; \node [title] (c) at (0,-.4) {$?$}; 
		\node [title] (d) at (2.6,0) {$\bs{|}$}; \node [title] (d) at (2.6,-.4) {$\Ord(\muu\r\log\d)$}; 
		
		\node [label] (imp) at (-3.3,.5) {impossible};
		\node [label] (known) at (3.3,.5) {well-studied};
		
		\node [label] (uncharted) at (-1.3,.75) {possible, but apparently};
		\node [label] (uncharted) at (-1.3,.5) {computationally};
		\node [label] (uncharted) at (-1.3,.25) {prohibitive};
		
		\node [label] (uncharted) at (1.3,.75) {possible, and apparently};
		\node [label] (uncharted) at (1.3,.5) {computationally};
		\node [label] (uncharted) at (1.3,.25) {feasible};
				
		\draw [font=\scriptstyle]
				(a) edge (z);
	\end{tikzpicture}
\end{center}


We first study the computational cost of missing data.  To this end we computed the minimum number of iterations required to complete a matrix, as a function of the number of uniform random samples per column $\L$.  The results are summarized in \numItersFig.  Unsurprisingly, the more missing data, the more iterations are required to complete the matrix.

\begin{figure}
\centering
\includegraphics[width=5cm]{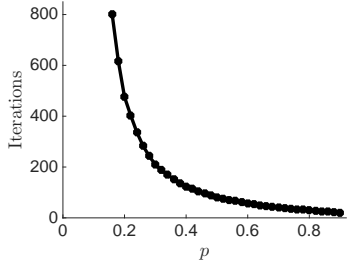}
\caption{Average number of iterations (over $500$ trials) required by \IHTSVD\ to complete a matrix with low coherence ($\muu<3$) with an accuracy of $10^{-12}$ (using normalized Frobenius norm), as a function of $\p:=\L/\d$, the proportion of uniform random samples per column, with ambient dimension $\d=500$ and rank $\r=10$.}
\label{numItersFig}
\end{figure}

In addition, we constructed samplings $\O$ with only $\L = \r + 1$ samples per column selected uniformly at random, and kept only those samplings satisfying the conditions of \LRMCCor, to guarantee that $\XO$ were uniquely completable (we used \validationAlg\ to determine whether each sampling satisfied these conditions).  Unfortunately, even though $\XO$ was uniquely completable, the matrix was incorrectly completed in every single trial.  This suggests that completion in this regime, now known to be theoretically possible (through the solution of the polynomial system $\F=\bs{0}$; see \proofSec), might be computationally prohibitive in practice.


We thus tested how much missing data can practical algorithms handle while remaining computationally efficient.
To this end, we sampled $\L < \muu\r \log\d$ entries per column, drawn uniformly at random, and ran the \IHTSVD\ algorithm for at most \phantomsection\label{TTDef}$\TT = \d$ iterations (see the beginning of Section~\ref{implicationsSec} for a discussion of this algorithmic choice). 

To truly test this regime, we considered a setup where previous theory would require all entries to be observed to guarantee a correct completion with probability at least $1-\eps$.  There are plenty of such scenarios.  We arbitrarily selected $\d=500$ and $\r=10$, and $\eps=1/\d$.

Our simulations, summarized in \samplesFig, show that practical algorithms tend to work consistently well with $\L < \muu\r\log(\frac{\fix{\d}{}}{\fix{\eps}{}})$.  This suggests that there is a regime with $\L <\muu\r\log(\frac{\fix{\d}{}}{\fix{\eps}{}})$ samples per column where completion is computationally feasible, even in the presence of noise.

\begin{figure}
\centering
\includegraphics[width=5cm]{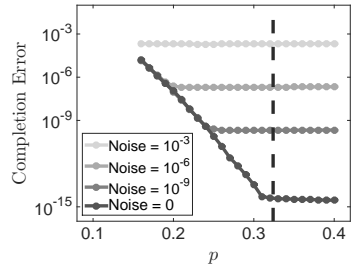}
\caption{Average completion error of \IHTSVD\ (over $500$ trials) for different levels of additive i.i.d.\ zero-mean Gaussian noise (noise variance as indicated in the legend), after at most $250$ iterations, as a function of $\p:=\L/\d$, the proportion of uniform random samples per column, with ambient dimension $\d=500$ and rank $\r=10$.    Previous guarantees would require all entries to be observed, and so in practice, existing theory would not allow one to confirm the correctness of a completion.  Our results do.  The dashed line represents $\p=\max\{12( \log(\frac{\fix{\d}{}}{\fix{\eps}{}})+1),2\r\}/\d$, with $\eps=\frac{1}{\fix{\d}{}}$, the sufficient condition of \probabilityThm, which implies that with probability at least $1-\eps$, for any $\p$ above this threshold, a rank-$\r$ completion is guaranteed to be correct, regardless of the completion method.}
\label{samplesFig}
\end{figure}


\subsection*{Dependence on Coherence Parameter}

In our next experiment, we study the practical role of coherence in \LRMC.  More precisely, we tested whether a matrix could be computationally efficiently completed as a function of its coherence parameter $\muu$, and the number of uniform random samples per column $\L$ (to generate matrices with a specific coherence parameter, we simply increased the magnitude of a few entries in $\Ustar$, until it had the desired coherence).  The results, summarized in \coherenceFig, suggest that for most of the coherence range, whether this algorithm can correctly complete the matrix mainly depends on the number of samples rather than on the coherence parameter.  For instance, see in \coherenceFig\ that given the number of samples, the success rate of this algorithm is about the same for most of the range of $\muu$.   Nonetheless, there are some cases with extremely large coherence ($\muu$ close to the maximum possible, $\frac{\fix{\d}{}}{\fix{\r}{}}$, corresponding to subspaces almost perfectly aligned with the canonical axes), where this algorithm tends to fail more often at reconstructing the matrix (these are cases where most of the information is concentrated in only a few entries, which brings computational and numerical accuracy problems).

\begin{figure}
\centering
\includegraphics[width=8cm]{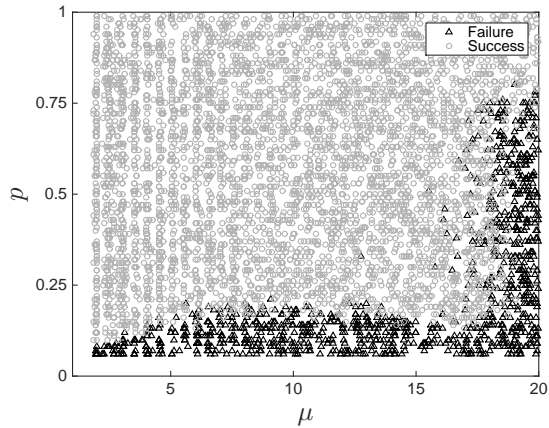}
\caption{Results of the \IHTSVD\ algorithm as a function of the coherence parameter $\muu \in [1,\frac{\fix{\d}{}}{\fix{\r}{}}]$ and the proportion of uniform random samples per column $\p:=\L/\d$, with ambient dimension $\d=100$ and rank $\r=5$.  Similar results were observed for other algorithms, including alternating minimization \cite{jain} and EM \cite{ssp14}.  In each of the $5,000$ trials we declared a success if the normalized completion error was below $10^{-12}$ (using normalized Frobenius norm).  Our theoretical results show that whether \ae\ matrix can be uniquely completed does not depend on its coherence.  This experiment suggests that in practice, this is also the case for most of the range of $\muu$.  For instance, given $\p$, the success rate of this algorithm is about the same for most of the range of $\muu$ (about $1 \leq \muu \leq 17$).  Nevertheless, the success rate quickly decays if the coherence is extremely high ($\muu$ close to the maximum possible, $\frac{\fix{\d}{}}{\fix{\r}{}}=20$).}
\label{coherenceFig}
\end{figure}

To further study the role of coherence in practice, we recorded the number of iterations that were required to complete each matrix (in the success cases of the previous experiment).  The results, summarized in \itersMuFig, suggest that while coherent matrices may be theoretically as completable as incoherent ones, in practice, the more coherent a matrix is, the more computationally expensive it may be to complete it.  Furthermore, the number of iterations seems to increase steadily for most of the coherence range, but after a transition point it suddenly seems to grow exponentially, suggesting the existence of a maximal coherence regime where, though theoretically possible, completion may be computationally impractical (similar to the minimal sampling regime from our previous experiment).

\begin{figure}
\centering
\includegraphics[width=5cm]{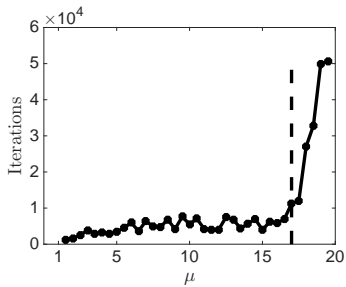}
\caption{Average number of iterations (of the success trials from \coherenceFig) required by \IHTSVD\ to complete a matrix with an accuracy of $10^{-12}$ (using normalized Frobenius norm), as a function of its coherence parameter $\muu$.  This suggests the existence of a maximal coherence regime (e.g., after the dashed line) where, though theoretically possible, completion may become computationally impractical.}
\label{itersMuFig}
\end{figure}

\subsection*{New Guarantees}
It is known that $\Ord(\muu \r \log\d)$ uniform random samples per column (with constants greater than $1$) are sufficient for completion \cite{candes-recht}.  There are non-pathological regimes (e.g., $\d=500$ and $\r=10$, or $\d=100$ and $\r=5$, and ideal coherence, as in our experiments) where these conditions end up requiring that all entries are observed.  Experiments show that the \IHTSVD\ algorithm can exactly complete such matrices when even fewer than half of the entries are observed, but prior theory gives no guarantees in these regimes, and so in practice, one would be unable to confirm the correctness of a completion.

Furthermore, typical conditions for \LRMC\ usually apply to matrices with bounded coherence, and require uniform random sampling with rates that depend on the coherence parameter $\muu$.  In many practical applications, sampling is hardly uniform (e.g., vision, where occlusion of objects produce missing data in very non-uniform random patterns), and $\muu$ is typically unknown, so the existing theory does not allow one to confirm the correctness of a completion.

Our results shed new light on these issues. \uniquenessThm\ states that regardless of coherence and the sampling model, if the observation pattern satisfies the conditions of the theorem, a rank-$\r$ completion, obtained by any method whatsoever, is guaranteed to be {\em the} correct completion.  In particular, \probabilityThm\ states that this will be the case with high probability under uniform random sampling models.

In some cases one can use \LRMCCor\ together with \validationAlg\ to verify efficiently and deterministically whether these conditions are satisfied.  Recall that \LRMCCor\ states that unique completability is possible if $\O$ contains $\r+1$ disjoint matrices $\{\Ot\}_{\fix{\ttau=1}}^{\fix{\r+1}}$, each of size $\d \times (\d-\r)$ satisfying \identifiabilityCond.  Given a matrix $\Ot$, \validationAlg\ allows to verify whether it satisfies \identifiabilityCond.  However, it provides no means to select the $\Ot$'s.

In general, one can construct samplings for which finding the right $\Ot$'s would require exponential time.  However, if the samples are well spread across the rows, then one may validate a completion deterministically by selecting the $\Ot$'s randomly.

To see this, suppose $\O$ has \phantomsection\label{NbreveDef}$(\r+1)\Nbreve$ columns, with $\Nbreve \geq \d-\r$ and exactly $\r+1$ observations per column (see \entriesRmk).  We can randomly partition $\O$ into $\r+1$ disjoint submatrices $\{\Obreve_\ttau\}_{\fix{\ttau=1}}^{\fix{\r+1}}$, each of size $\Nbreve$.  One can then use \validationAlg\ to verify whether each $\Obreve_\ttau$ contains an $\d \times (\d-\r)$ submatrix $\Ot$ satisfying \identifiabilityCond.  If this is the case, then we know deterministically that the completion is correct.

\identifiabilityCondFig\ shows that as $\Nbreve$ grows, the probability that each $\Obreve_\ttau$ contains an $\Ot$ satisfying \identifiabilityCond\ quickly approaches $1$.  For example, with $\Nbreve$ as small as $2(\d-\r)$, i.e., with only twice as many columns as strictly necessary, each $\Obreve_\ttau$ will contain an $\Ot$ satisfying \identifiabilityCond\ with probability larger than $.999$.  This suggests that if we find a low-rank completion of a matrix, we can expect that a random partition will certify it through \LRMCCor\ and \validationAlg.
    
\begin{figure}
\centering
\includegraphics[width=5cm]{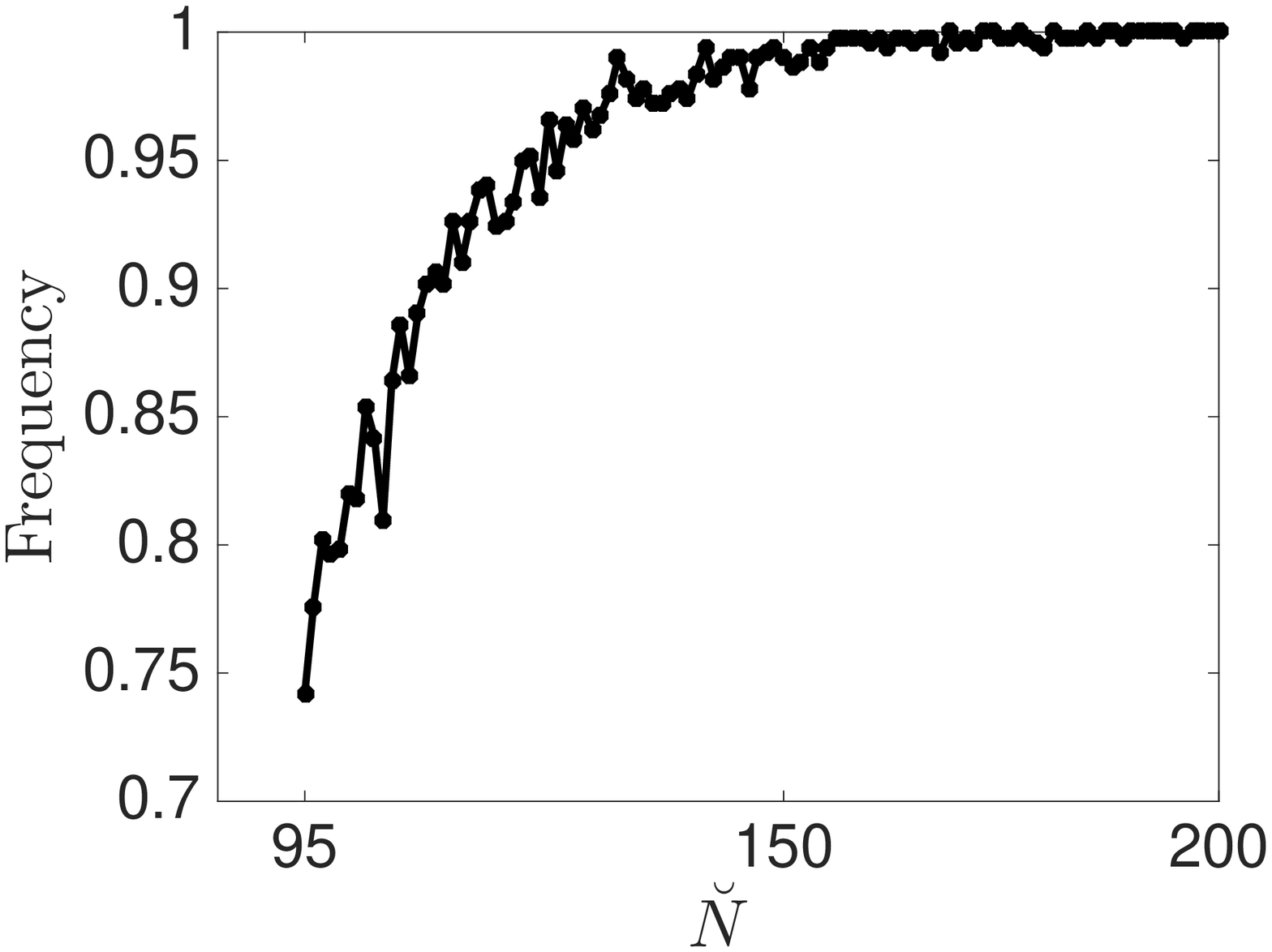}
\caption{We generated $\d \times \Nbreve$ matrices $\Obreve$ with only $\r+1$ samples per column, selected uniformly at random, with $\d=100$, $\r=5$.  This figure shows the proportion of times (over $500$ trials) that $\Obreve$ contains a $\d \times (\d-\r)$ matrix $\Ohat$ satisfying \identifiabilityCond, as a function of $\Nbreve$.  We used \validationAlg\ to determine whether this was the case.  Notice that as $\Nbreve$ grows, the probability that each $\Obreve$ contains an $\Ot$ satisfying \identifiabilityCond\ quickly approaches $1$.}
\label{identifiabilityCondFig}
\end{figure}

This way, our results can be used to certify the correctness of a completion, thus bringing guarantees applicable to any algorithm, under any sampling model, in lieu of coherence assumptions.

\subsection*{Adaptive Sampling}
If one could select which entries of $\X$ to observe, perhaps the easiest way to recover $\X$ is to sample $\r$ linearly independent columns to obtain a basis of the subspace, and then $\r$ rows to obtain the coefficients of each column in this basis. However, in many \LRMC\ applications, the entries one may observe can be limited.  Take for example recommender systems, where obtaining a complete column equates to asking a single user (column) to evaluate every item (row).  In these problems the number or rows can be very large, hence this can be an unreasonable thing to ask.  Moreover, the combinations of rows that one may sample could be restricted.  An other example arises in distributed settings, where at each location one may only sample certain subsets of all the information.

Our results tell us exactly which entries to look for.  Furthermore, it is fairly simple to construct sampling patterns that satisfy the conditions of Theorems \ref{LRMCThm} and \ref{uniquenessThm} and \LRMCCor\ that do not require to sample full columns or rows.  For instance, we can generate random samplings, use \validationAlg\ to verify whether they satisfy condition \identifiabilityCond\ (most of them will; see \identifiabilityCondFig), and keep them or discard them depending on this.

Deterministic constructions are also possible.  For instance, it is easy to verify that each of the blocks in \adaptiveEg\ satisfies \identifiabilityCond, which implies $\O$ satisfies the conditions of \LRMCCor, and can thus be uniquely completed.  This example corresponds to asking the $\i^{\rm th}$ user of each block to rate items $\i$ through $\i+\r$.  We conclude that if the entries one may choose to observe are limited (as is the case in many \LRMC\ applications), one can directly apply our results to adaptively design observation patterns that guarantee completability.

\section{Proof of \LRMCThm}
\label{proofSec}

For any subspace, matrix or vector that is compatible with a set of indices \phantomsection\label{oDef}$\o$, we will use the subscript \phantomsection\label{cdotoDef}$\o$ to denote its restriction to the coordinates/rows in $\o$.  For example, letting \phantomsection\label{oiDef}$\oi$ denote the indices of the nonzero rows of the $\i^{\rm th}$ column of $\O$, then \phantomsection\label{xoiDef}$\xoi \in \R^{\fix{\r+1}}$ and \phantomsection\label{sstaroiDef}$\sstaroi \subset \R^{\fix{\r+1}}$ denote the restrictions of the $\i^{\rm th}$ column in $\X$ and $\sstar$, to the indices in $\oi$.  We say that an $\r$-dimensional subspace $\s$ \phantomsection\label{fitsDef}{\em \fits} $\XO$ if $\xoi \in \s_{\fix{\oi}}$ $\forall \i$.

\subsection*{The Variety $\SS$}
Let us start by studying the variety of all $\r$-dimensional subspaces that \fit\ $\XO$.  First observe that in general, the restriction of an $\r$-dimensional subspace to $\L \leq \r$ coordinates is $\R^{\fix{\L}}$.  We formalize this in the following definition, which essentially states that a subspace is \nondegenerate\ if its restrictions to $\L \leq \r$ coordinates are $\R^{\fix{\L}}$.

\begin{myDefinition}[Degenerate subspace]
\label{degenerateDef}
We say $\s \in \Gr(\r,\R^{\fix{\d}})$ is {\em \degenerate} if and only if there exists a set $\o \subset \{1,\dots,\d\}$ with $|\o | \leq\r$, such that $\dim \s_{\fix{\o}}< | \o |$.
\end{myDefinition}

Let $\nuuG$ denote the uniform measure on $\Gr(\r,\R^{\fix{\d}})$.  A subspace is \degenerate\ if and only if an $\r \times \r$ submatrix of one of its bases is rank-deficient. This is equates to having a zero determinant.  Since the determinant is a polynomial in the entries of a matrix, this is a condition of $\nuuG$-measure zero.

Since $\nuuG$-almost every subspace is \nondegenerate, let us consider only the subspaces in \phantomsection\label{GrNDDef}$\GrND(\r,\R^{\fix{\d}}) \subset \Gr(\r,\R^{\fix{\d}})$, the set of all \nondegenerate\ $\r$-dimensional subspaces of $\R^{\fix{\d}}$.

Define \phantomsection\label{SSDef}$\SS(\XO)\subset \GrND(\r,\R^{\fix{\d}})$ such that every $\s \in \SS(\XO)$ \fits\ $\XO$, i.e.,
\begin{align*}
\SS(\XO) \ := \ \Big\{\s \in \GrND(\r,\R^{\fix{\d}}) \ : \ \{\xoi \in \s_{\fix{\oi}}\}_{\fix{\i=1}}^{\fix{\N}} \Big\}.
\end{align*}

Let \phantomsection\label{UDef}$\U \in \R^{\fix{\d \times \r}}$ be a basis of $\s \in \SS(\XO)$.  The condition $\xoi \in \s_{\fix{\oi}}$ is equivalent to saying that there exists a vector \phantomsection\label{tiDef}$\ti \in \R^{\fix{\r}}$ such that
\begin{align}
\label{xUthetaEq}
\xoi \ = \ \U_{\fix{\oi}} \ti.
\end{align}

We can see that if $\xoi$ has fewer than $\r$ observations, \eqref{xUthetaEq} will be an underdetermined system with infinitely many solutions, and hence $\xoi$ can be completed in infinitely many ways.

If $\xoi$ has exactly $\r$ observations, \eqref{xUthetaEq} becomes a system with $\r$ equations and $\r$ unknowns (the elements of $\ti$).  This will be the case for every $\s \in \GrND(\r,\R^{\fix{\d}})$.  Hence a column with exactly $\r$ observations can be uniquely completed once $\sstar$ is known, but it provides no information to identify $\sstar$.

On the other hand, if $\xoi$ has exactly $\r+1$ observations, then \eqref{xUthetaEq} becomes an overdetermined system with $\r+1$ equations and $\r$ unknowns.  This imposes one constraint on the elements of $\U_{\fix{\oi}}$, thus restricting the set of subspaces that \fit\ $\xoi$.

In general, each column with $\r+1$ observations will impose one constraint that may reduce one of the $\r(\d-\r)$ degrees of freedom in $\GrND(\r,\R^{\fix{\d}})$.  Therefore, one necessary condition for completion is that $\XO$ imposes at least $\r(\d-\r)$ constraints.

We will now study these constraints and characterize when exactly will they reduce all the $\r(\d-\r)$ degrees of freedom in $\GrND(\r,\R^{\fix{\d}})$, thus restricting $\SS(\XO)$ to a set with at most finitely many elements.

Let \phantomsection\label{deltaiDef}\phantomsection\label{nabliDef}$\{\deltai,\nabli\}$ be a partition of the $\r+1$ elements of $\oi$, such that $\deltai$ has exactly $\r$ elements, and $\nabli$ has only one element.  We can then expand \eqref{xUthetaEq} as
\begin{align*}
\begin{matrix}
\r \left. \begin{matrix} \\ \\ \\ \end{matrix} \right\{ \\
1 \left. \begin{matrix} \\ \end{matrix} \right\{
\end{matrix}
\left[ \begin{matrix}
\multirow{3}{*}{$\x_{\fix{\deltai}}$} \\ \\ \\ \hline {\rm x}_{\fix{\nabli}} \end{matrix} \right]
\ = \
\left[ \begin{array}{c}
\multirow{3}{*}{$\Scale[1.5]{\U_{\fix{\deltai}}}$} \\ \\ \\ \hline
\U_{\fix{\nabli}}
\end{array}\right] \ti.
\end{align*}

Since $\s$ is \nondegenerate, $\U_{\fix{\deltai}}$ is full-rank, so we may solve for $\ti$ using the top block to obtain $\ti = \U^{-1}_{\fix{\deltai}} \x_{\fix{\deltai}}$.  Plugging this on the last row, we have that \eqref{xUthetaEq} is equivalent to:
\begin{align}
\label{xuicEq}
{\rm x}_{\fix{\nabli}} \ = \ \U_{\fix{\nabli}} \U^{-1}_{\fix{\deltai}} \x_{\fix{\deltai}}.
\end{align}

On the other hand, $\xoi$ lies in $\sstaroi$ by assumption.  This implies that there exists a unique \phantomsection\label{tiStarDef}$\tiStar \in \R^{\fix{\r}}$ such that
\begin{align}
\label{tiStarEq}
\xoi \ = \ \Ustaroi \tiStar,
\end{align}
where $\Ustar$ is a basis of $\sstar$.  Substituting \eqref{tiStarEq} in \eqref{xuicEq} we obtain
\begin{align}
\label{nonpolyEq}
\UstarNabli \tiStar \ = \ \U_{\fix{\nabli}} \U_{\fix{\deltai}}^{-1} \UstarDeltai \tiStar.
\end{align}

Recall that $\U^{-1}_{\fix{\deltai}} = \U_{\fix{\deltai}}^{\fix{\adj}} / |\U_{\fix{\deltai}}|$, where \phantomsection\label{adjDef}$\U_{\fix{\deltai}}^{\fix{\adj}}$ and $|\U_{\fix{\deltai}}|$ denote the adjugate and the determinant of $\U_{\fix{\deltai}}$.  Therefore, we may rewrite \eqref{nonpolyEq} as the following polynomial equation:
\begin{align}
\label{polyEq}
\Big( |\U_{\fix{\deltai}}| \UstarNabli - \U_{\fix{\nabli}} \U^{\fix{\adj}}_{\fix{\deltai}} \UstarDeltai \Big) \tiStar \ = \ 0.
\end{align}

We conclude that a subspace $\s$ with basis $\U$ \fits\ $\XO$ if and only if $\U$ satisfies \eqref{polyEq} for every $\i=1,\dots,\N$.

Since every nontrivial subspace has infinitely many bases, even if there is only one $\r$-dimensional subspace in $\SS(\XO)$, the variety
\begin{align*}
\left\{ \ \U \in \R^{\fix{\d \times \r}} \ : \ \big( |\U_{\fix{\deltai}}| \UstarNabli - \U_{\fix{\nabli}} \U^{\fix{\adj}}_{\fix{\deltai}} \UstarDeltai \big) \tiStar = 0 \hspace{.5cm} \forall \ \i \ \right\}
\end{align*}
has infinitely many solutions.  Therefore, we will associate a unique $\U$ with each subspace as follows.  Observe that for every $\s \in \GrND(\r,\R^{\fix{\d}})$, we can write $\s=\spn\{\U\}$ for a unique $\U$ in the following column echelon form:
\begin{align}
\label{columnEchelonEq}
\U \ = \
\left[ \begin{array}{c}
\multirow{2}{*}{
\Scale[1.5]{\I}}\hspace{.1cm} \\ \\ \hline
\multirow{3}{*}{\Scale[1.5]{\V}} \\ \\ \\
\end{array}\right]
\begin{matrix}
\left. \begin{matrix} \\ \\ \end{matrix} \right\} \r \hspace{.7cm} \\
\left. \begin{matrix} \\ \\ \\ \end{matrix} \right\} \d-\r.
\end{matrix}
\end{align}

On the other hand, every $\V \in \R^{\fix{(\d-\r) \times \r}}$ defines a unique $\r$-dimensional subspace of $\R^{\fix{\d}}$, via $\spn\{\U\}$.  Moreover, $\spn\{\U\}$ will be \nondegenerate\ for almost every $\V$, with respect to \phantomsection\label{nuuVDef}$\nuuV$: the Lebesgue measure on $\R^{(\fix{\d-\r ) \times \r}}$.  Let \phantomsection\label{RNDDef}$\RND \subset \R^{(\fix{\d-\r) \times \r}}$ denote the set of all $(\d-\r) \times \r$ matrices $\V$ whose $\spn\{\U\}$ is \nondegenerate, or equivalently, whose $\r \times \r$ submatrices of $\U$ are full-rank.  Then we have a bijection between $\GrND(\r,\R^{\fix{\d}})$ and $\RND$ via $\sstar=\spn\{\U\}$.   It follows that a statement holds for  $(\nuuG \times \nuuT)$-almost every pair $\{\sstar,\Tstar\}$ if and only if it holds for  $(\nuuV \times \nuuT)$-almost every pair $\{\Vstar,\Tstar\}$. We will use these measures interchangeably.

\subsection*{The Set $\F$}
Continuing with our analysis, recall that a subspace $\s$ with basis $\U$ will \fit\ $\XO$ if and only if $\U$ satisfies \eqref{polyEq} for every $\i$.  With this in mind, define
$$\phantomsection\label{efiDef}
\efi(\V | \Vstar,\tiStar) \ := \ \Big( |\U_{\fix{\deltai}}| \UstarNabli - \U_{\fix{\nabli}} \U^{\fix{\adj}}_{\fix{\deltai}} \UstarDeltai \Big) \tiStar,
$$
with $\U$ and $\Ustar$ in the column echelon form in \eqref{columnEchelonEq}.  We will use $\efi$ as shorthand, with the understanding that $\efi$ is a polynomial in the elements of $\V$, and that the elements of $\Vstar$ and $\tiStar$ play the role of coefficients.

Furthermore, let
$$
\phantomsection\label{FDef}
\F(\V | \Vstar, \Tstar) \ := \ \left\{ \efi \right\}_{\fix{\i}=1}^{\fix{\N}},
$$
and use $\F(\V)$, or simply $\F$ as shorthand, with the understanding that $\F$ is a set of polynomials in the elements of $\V$, and that the elements of $\Vstar$ and $\Tstar$ play the role of coefficients.  We will also use $\F=\bs{0}$ as shorthand for $\{ \efi = 0\}_{\fix{\i}=1}^{\fix{\N}}$.

This way, we may rewrite:
\begin{align*}
\SS(\XO) \ = \ \left\{ \spn\left[ \begin{matrix}\I \\ \V \end{matrix} \right] \in \GrND(\r,\R^{\fix{\d}}) \ : \ \F(\V) =\bs{0} \right\}.
\end{align*}

In general, the affine variety
$$
\phantomsection\label{VVDef}
\VV(\F) \ := \ \left\{ \V \in \RND \ : \ \F(\V) =\bs{0} \right\}
$$
could contain an infinite number of elements.  We are interested in conditions that guarantee there is only one or (slightly less demanding) only a finite number.  The following lemma states that this will be the case if and only if $\r(\d-\r)$ polynomials in $\F$ are algebraically independent.

\begin{myLemma}
\label{dimensionLem}
Let \Aone\ hold.  For \ae\ $\X$, $\SS(\XO)$ contains at most finitely many subspaces if and only if $\r(\d-\r)$ polynomials in $\F$ are algebraically independent.
\end{myLemma}
\begin{proof}
By our previous discussion, for \ae\ $\X$ there are at most finitely many subspaces in $\SS(\XO)$ if and only if there are at most finitely many points in $\VV(\F)$.  We know from algebraic geometry that this will be the case if and only if $\dim \VV(\F)=0$ (see, e.g., Proposition 6 in Chapter 9, Section 4 of \cite{cox}).

Since $\VV(\F) \subset \RND$, we know that if $\dim \VV(\F)=0$, then $\F$ must contain $\r(\d-\r)$ algebraically independent polynomials (see, e.g., Exercise 16 in Chapter 9, Section 6 of \cite{cox}).

On the other hand, we know that $\dim \VV(\F)=0$ if $\r(\d-\r)$ polynomials in $\F$ are a regular sequence (see, e.g., Exercise 8 in Chapter 9, Section 4 of \cite{cox}).

Finally, since being a regular sequence is an open condition, it follows that for $(\nuuV \times \nuuT)$-almost every $\{\Vstar,\Tstar\}$, polynomials in $\F$ are algebraically independent if and only if they are a regular sequence (see, e.g., Remark 3.4 in \cite{aramova}).
\end{proof}

\begin{myRemark}
The next part of our analysis studies conditions to guarantee that the polynomials in $\F$ are algebraically independent.  Following up on \entriesRmk, any observation, in addition to the $\r+1$ per column that we assume, cannot increase the number of subspaces that agree with the observations.  In effect, each observed entry, in addition to the first $\r+1$ observations, places one additional polynomial constraint analogous to $\efi$.  However, the polynomials produced by the same column share the same coefficient $\tiStar$.  Intuitively, this means that the polynomials are no longer generic.  While these polynomials might or might not be algebraically dependent, in general it is difficult to determine which is the case.

For this reason we assume \Aone: that each column is observed on exactly $\r+1$ entries.  This way the $\r+1$ entries in each column produce only one polynomial constraint.  This guarantees that we only use one polynomial per column, so that all the coefficients of the polynomials in $\F$ are generic, and easier to study.  In general, if some columns are observed on more than $\r+1$ entries, all we need is that the observed entries contain a pattern with exactly $\r+1$ observations per column satisfying our sampling conditions.
\end{myRemark}

\subsection*{Algebraic Independence}
By the previous discussion, there are at most finitely many $\r$-dimensional subspaces that \fit\ $\XO$ if and only if there is a subset \phantomsection\label{FtildeDef}$\Ftilde$ of $\r(\d-\r)$ polynomials in $\F$ that is algebraically independent.

Whether this is the case depends on the supports of the polynomials in $\Ftilde$, i.e., on $\Otilde$: the subset of columns in $\O$ corresponding to such polynomials.  \independenceLem\ shows that the polynomials in $\Ftilde$ will be algebraically independent if and only if $\Otilde$ satisfies the conditions in \LRMCThm.

\begin{myLemma}
\label{independenceLem}
Let \Aone\ hold.  For \ae\ $\X$, the polynomials in $\Ftilde$ are algebraically dependent if and only if $\nOf(\Op) > \r(\mOf(\Op)-\r)$ for some matrix $\Op$ formed with a subset of the columns in $\Otilde$.
\end{myLemma}

To show this statement we will use Lemmas \ref{LLem} and \ref{basisLem} below.

Let $\Op$ be a subset of the columns in $\Otilde$, and let \phantomsection\label{FpDef}$\Fp$ be the subset of the $\nOf(\Op)$ polynomials in $\Ftilde$ corresponding to such columns.  Notice that $\Fp$ only involves the variables in $\U$ corresponding to the $\mOf(\Op)$ nonzero rows of $\Op$.

Let \phantomsection\label{liDef}$\li(\Op)$ be the largest number of algebraically independent polynomials in $\Fp$.

\begin{myLemma}
\label{LLem}
For \ae\ $\X$, $\li(\Op) \leq \r(\mOf(\Op)-\r)$.
\end{myLemma}
\begin{proof}
Observe that the column echelon form in \eqref{columnEchelonEq} was chosen arbitrarily.  As a matter of fact, for every permutation of rows \phantomsection\label{PiiDef}$\Pii$ and every $\s \in \GrND(\r,\R^{\fix{\d}})$, we may write $\s=\spn\{\U\}$, for a unique $\U$ in the following permuted column echelon form:
\begin{align*}
\U \ = \ \Pii \left[ \begin{matrix}\I \\ \V \end{matrix} \right].
\end{align*}
For example, we could take $\Pii$ to swap the top and bottom blocks in \eqref{columnEchelonEq}, and take $\U$ in the following form:
\begin{align*}
\U \ = \ \Pii \left[ \begin{matrix}\I \\ \V \end{matrix} \right] \ = \ \left[ \begin{matrix}\V \\ \I \end{matrix} \right].
\end{align*}
Observe that in general, $\U$, $\V$ and $\F$ will be different for each choice of $\Pii$.  Nevertheless, the condition $\xoi \in \s_{\fix{\oi}}$ is invariant to the choice of basis of $\s$.  This implies that while different choices of $\Pii$ produce different $\F$'s, the variety
\begin{align*}
\SS(\XO) \ = \ \left\{ \spn \ \Pii \left[ \begin{matrix}\I \\ \V \end{matrix} \right] \in \GrND(\r,\R^{\fix{\d}}) \ : \ \F(\V) =\bs{0} \right\}
\end{align*}
is the same for every $\Pii$.

This implies that the number of algebraically independent polynomials in $\Fp$ is invariant to the choice of $\Pii$.  Therefore, showing that \LLem\ holds for one particular $\Pii$ suffices to show that it holds for every $\Pii$.

With this in mind, take $\Pii$ such that $\U$ is written with the identity block in the position of $\r$ nonzero rows of $\Op$.

Since the polynomials in $\Fp$ only involve the elements of the $\mOf(\Op)$ rows of $\U$ corresponding to the nonzero rows of $\Op$, and $\U$ has the identity block in the position of $\r$ nonzero rows of $\Op$, it follows that the polynomials in $\Fp$ only involve the $\r (\mOf(\Op)-\r)$ variables in the $\mOf(\Op)-\r$ corresponding rows of $\V$.  Furthermore, $\Fp=\bs{0}$ has at least one solution.  This implies $\li(\Op) \leq \r(\mOf(\Op)-\r)$, as desired.
\end{proof}

We say $\Fp$ is \phantomsection\label{madDef}{\em \mad} if the polynomials in $\Fp$ are algebraically dependent, but every proper subset of the polynomials in $\Fp$ is algebraically independent.

\begin{myLemma}
\label{basisLem}
For \ae\ $\X$, if $\Fp$ is \mad, then $\nOf(\Op) = \r(\mOf(\Op)-\r)+1$.
\end{myLemma}

In order to prove \basisLem\ we will need the next two lemmas.

\begin{myLemma}
\label{oneFixedVarLem}

Take $\Pii$ such that $\U_{\fix{\deltai}} = \U^\star_{\fix{\deltai}} =\I$.  For \ae\ $\X$, if \phantomsection\label{FdpDef}$\Fp=\{\Fdp,\efi\}$ is \mad, then all solutions to $\Fp=\bs{0}$ satisfy $\U_{\fix{\nabli}}=\bs{{\rm U}}^\star_{\fix{\nabli}}$.
\end{myLemma}

The intuition behind this lemma is as follows: suppose for contrapositive that there are infinitely many solutions $\U_{\fix{\oi}}$ to $\Fdp=\bs{0}$ with $\U_{\fix{\deltai}}=\I$.  Each of these solutions defines a different subspace.  Since $\{\Fdp,\efi\}$ is \mad, \ae\ solution to $\Fdp$ must \fit\ $\x_{\fix{\oi}}$.  This will only happen if $\x_{\fix{\oi}}$ lies in the intersection of infinitely many $\r$-dimensional subspaces, which is at most $(\r-1)$-dimensional.  But since $\x_{\fix{\oi}}$ is drawn from $\sstar$ (an $\r$-dimensional subspace), we know that almost surely $\x_{\fix{\oi}}$ will not lie in such $(\r-1)$-dimensional subspace.

\begin{proof}
Suppose that $\Fp=\{\Fdp,\efi\}$ is \mad, and let \phantomsection\label{vvDef}$\vv_{i}$ denote the row of $\V$ corresponding to $\U_{\fix{\nabli}}$, such that $\efi$ simplifies into
\begin{align*}
\efi(\vv_{i}, \U_{\fix{\deltai}} | \Vstar, \tiStar) \ &= \ \Big( |\U_{\fix{\deltai}}| \vviStar - \vv_{i} \U_{\fix{\deltai}}^{\fix{\adj}} \U^\star_{\fix{\deltai}} \Big) \tiStar \\
&= \ \left( \vviStar - \vv_{i} \right) \tiStar.
\end{align*}
Since $\efi$ involves $\vv_{i}$, $\Fdp$ must contain at least one polynomial in $\vv_{i}$ (otherwise $\Fp$ cannot be \mad).  This means that $\Fdp$ contains at least one polynomial \phantomsection\label{efjDef}$\efj$ involving $\vv_{i}$:
$$
\phantomsection\label{tjStarDef}
\efj(\vv_{i}, \U_{\fix{\deltaj}} | \Vstar, \tjStar) \ = \ \Big( |\U_{\fix{\deltaj}}| \vviStar - \vv_{i} \U_{\fix{\deltaj}}^{\fix{\adj}} \bs{{\rm U}}^\star_{\fix{\deltaj}} \Big) \tjStar.
$$
For \ae $\X$, $\tjStar$ is independent of $\tiStar$, so $(\nuuV \times \nuuT)$-almost surely, $\efi \neq \efj$.

We want to show that if $\Fp$ is \mad, then $\vv_{i}=\vviStar$ is the only solution to $\Fp=\bs{0}$.  So define \phantomsection\label{vjNumDef}$\vv_{i} =: [\vjNum{i1} \ \wwj ]$, and assume for contradiction that there exists a solution to $\Fdp=\bs{0}$ with \phantomsection\label{wjDef}$\wwj = \wj \neq \wwjStar$ and \phantomsection\label{GammajDef}$\U_{\fix{\deltaj}}=\Gammaj$, that is also a solution to $\Fp=\bs{0}$.

Next consider the univariate polynomials in \phantomsection\label{giDef}\phantomsection\label{gjDef}$\vjNum{i1}$ evaluated at this solution:
\begin{align*}
\gi(\vjNum{i1} | \Vstar, \tiStar) \ &:= \ \efi ( \vjNum{i1},\wwj, \U_{\fix{\deltai}} | \Vstar, \tiStar ) \Big|_{\fix{\wwj}=\fix{\wj}, \fix{\U}_{\fix{\deltai}}=\fix{\I}}, \\
\gj(\vjNum{i1} | \Vstar, \tjStar) \ &:= \ \efj ( \vjNum{i1},\wwj, \U_{\fix{\deltaj}} | \Vstar, \tjStar ) \Big|_{\fix{\wwj}=\fix{\wj}, \fix{\U}_{\fix{\deltaj}}=\fix{\Gammaj}},
\end{align*}
and observe that since $\{\wj, \Gammaj \}$ are a solution to $\Fp$, then $\gi$ and $\gj$ must have a common root.

We know from elimination theory that two distinct polynomials $\gi,\gj$ have a common root if and only if their resultant $\Res(\gi,\gj)$ is zero (see, for example, Proposition 8 in Chapter 3, Section 5 of \cite{cox}).

But $\Res(\gi,\gj)$ is a polynomial in the coefficients of $\gi$ and $\gj$.  In other words, \phantomsection\label{hhDef}$\Res(\gi,\gj) = \hh(\Vstar,\tiStar,\tjStar)$ for some nonzero polynomial $\hh$ in $\Vstar$, $\tiStar$ and $\tjStar$.  Therefore, $\hh \neq 0$ for $(\nuuV \times \nuuT)$-almost every $\{\Vstar,\Tstar\}$ (since the variety defined by $\hh=0$ has measure zero). Equivalently, $\hh \neq 0$ for \ae\ $\X$.  Since $\Res(\gi,\gj) \neq 0$, it follows that $\gi$ and $\gj$ do not have a common root $\vjNum{i1}$, which is the desired contradiction.

This will be true for either almost every $\wj$ in an infinite collection, or for every $\wj$ in a finite collection.  In the first case, we would conclude that $\Fp=\bs{0}$ has infinitely fewer solutions than $\Fdp=\bs{0}$, in contradiction to the \mad\ assumption.  In the second case, we conclude that $\wwjStar$ is the only solution to $\Fp=\bs{0}$.

Since $\vjNum{i1}$ was an arbitrary entry of $\U_{\fix{\oi}}$, we conclude that for \ae\ $\X$, if $\Fp$ is \mad, then $\U_{\fix{\nabli}} = \bs{{\rm U}}^\star_{\fix{\nabli}}$ is the only solution to $\Fp=\bs{0}$, as desired.
\end{proof}

Define \phantomsection\label{VnumDef}$\{\Vnum{t}, \Vcnum{t}\}$ as the partition of the variables involved in the polynomials in \phantomsection\label{FpNumDef}$\FpNum{t} \subset \Fp$, such that all the variables in $\Vnum{t}$ are uniquely determined by $\Fp=\bs{0}$.

\begin{myLemma}
\label{allFixedVarsLem}
Suppose $\Vnum{t} \neq \emptyset$ and that every $\efi \in \FpNum{t}$ is a polynomial in at least one of the variables in $\Vnum{t}$.  Then for \ae\ $\X$, all the variables involved in $\FpNum{t}$ are uniquely determined by $\Fp=\bs{0}$.
\end{myLemma}

\begin{proof}
Let \phantomsection\label{vcDef}$\vc$ be one of the variables in $\Vcnum{t}$ and let $\efi$ be a polynomial in $\FpNum{t}$ involving $\vc$.  By assumption on $\FpNum{t}$, $\efi$ also involves at least one of the variables in $\Vnum{t}$, say \phantomsection\label{vDef}$\v$.

Let \phantomsection\label{wDef}$\w$ denote the set of all variables involved in $\efi$ except $\v$.  Observe that $\vc \in \w$.  This way, $\efi$ is shorthand for $\efi(\v, \w | \Vstar, \tiStar)$.

We will show that for \ae\ $\X$, all the variables in $\w$ are also uniquely determined by $\Fp=\bs{0}$.

Suppose there exists a solution to $\Fp=\bs{0}$ with \phantomsection\label{voneDef}$\w = \vone$, and define the univariate polynomial
$$
\phantomsection\label{gDef}
\g(\v | \Vstar, \tiStar) \ := \ \efi ( \v,\w | \Vstar, \tiStar ) \Big|_{\fix{\w}=\fix{\vone}} \ .
$$
Now assume for contradiction that there exists another solution to $\Fp=\bs{0}$ with $\w \neq \vone$.  Let \phantomsection\label{vtwoDef}$\w = \vtwo$ be an other solution to $\Fp=\bs{0}$, and define
$$
\phantomsection\label{gpDef}
\gp(\v | \Vstar, \tiStar) \ := \ \efi ( \v,\w | \Vstar, \tiStar ) \Big|_{\fix{\w}=\fix{\vtwo}} \ .
$$

We will first show that $\g \neq \gp$.  To see this, recall the definition of $\efi$, and observe that it depends on the choice of $\nabli$.  Nevertheless, it is easy to see that $\efi=0$ describes the same variety regardless of the choice of $\nabli$.  Intuitively, this means that even though $\efi$ might look different for each choice of $\nabli$, it really {\em is} the same.

Therefore, we may select $\nabli$ to be the element of $\oi$ corresponding to the position of a variable of $\w$ that takes different values in $\vone$ and $\vtwo$.  This way, a variable with multiple solutions is located in the location of $\U_{\fix{\nabli}}$.  Since $\efi$ is linear in $\U_{\fix{\nabli}}$, it follows that $\g \neq \gp$ for $(\nuuV \times \nuuT)$-almost every $\{\Vstar,\Tstar\}$.

Now observe that since $\v$ is uniquely determined by $\Fp=\bs{0}$, $\g$ and $\gp$ have a common root, which immediately implies that there are at most finitely many distinct $\gp$.  Otherwise, $\v$ would be a common root to infinitely many distinct polynomials, which $(\nuuV \times \nuuT)$-almost surely cannot be the case.

We know from elimination theory that two distinct polynomials $\g,\gp$ have a common root if and only if their resultant $\Res(\g,\gp)$ is zero (see, for example, Proposition 8 in Chapter 3, Section 5 of \cite{cox}).

But $\Res(\g,\gp)$ is a polynomial in the coefficients of $\g$ and $\gp$.  In other words, \phantomsection\label{hDef}$\Res(\g,\gp) = \h(\Vstar, \tiStar)$ for some nonzero polynomial $\h$ in $\Vstar$ and $\tiStar$.  Therefore, $\h \neq 0$ for $(\nuuV \times \nuuT)$-almost every $\{\Vstar,\Tstar\}$ (since the variety defined by $\h=0$ has measure zero).    Equivalently, $\h \neq 0$ for \ae\ $\X$.

Since $\Res(\g,\gp) \neq 0$, it follows that $\g$ and $\gp$ do not have a common root $\v$, which is the desired contradiction.  This is true for all of the finitely many $\gp$.  This shows that  for \ae\ $\X$, all the variables in $\w$ (including $\vc$) are uniquely determined by $\Fp=\bs{0}$.

Since $\vc$ was an arbitrary element in $\Vcnum{t}$, we conclude that all the variables in $\Vcnum{t}$ are uniquely determined by $\Fp=\bs{0}$.
\end{proof}

With this, we are now ready to present the proofs of \basisLem, \independenceLem\ and \LRMCThm.

\begin{proof}(\basisLem)
By the same arguments as in \LLem, whether $\Fp$ is \mad\ is invariant to any permutation $\Pii$ of the rows of the column echelon form in \eqref{columnEchelonEq}.  Therefore, showing that \basisLem\ holds for one particular choice of $\Pii$ suffices to show it holds for every $\Pii$.

With this in mind, suppose $\Fp=\{\Fdp,\efi\}$ is \mad.  Take $\Pii$ such that $\U$ and $\Ustar$ are written in the column echelon form in \eqref{columnEchelonEq} with the identity block in the rows indexed by $\deltai$, and let $\vv_{i}$ denote the row of $\V$ corresponding to $\U_{\fix{\nabli}}$, such that
\begin{align*}
\U_{\fix{\oi}} \ = \ \left[ \begin{array}{c}
\\ \vspace{.25cm} \hspace{.2cm} \Scale[1.5]{\I} \hspace{.2cm}  \\ \hline
\vv_{i}
\end{array}\right]
\begin{matrix}
\left. \begin{matrix} \\ \vspace{.2cm} \\ \end{matrix} \right\} \r \\
\left. \begin{matrix} \\ \end{matrix} \right\} 1.
\end{matrix}
\end{align*}

\phantomsection\label{eighthDef}We know by \oneFixedVarLem\ that $\vv_{i}$ is uniquely determined by $\Fp=\bs{0}$.  We will now iteratively use \allFixedVarsLem\ to show that all the variables in $\Fp$ (which are the same as the variables in $\Fdp$) are also uniquely determined by $\Fp=\bs{0}$.  This will imply that all the variables in $\Fdp$ are finitely determined by $\Fdp=\bs{0}$, and that $\Fdp$ contains the same number of polynomials, $\nOf(\Odp)$, as variables, $\r(\mOf(\Odp)-\r)$, which is the desired conclusion.

First observe that since $\vv_{i}$ is finitely determined by $\Fdp=\bs{0}$, $\Fdp$ must contain at least $\r$ polynomials in $\vv_{i}$.  Denote these polynomials by $\FpNum{1} \subset \Fdp$.

We will proceed inductively, indexed by \phantomsection\label{tDef}$\t\geq 1$.
First, set $\t=1$ and define $\Vnum{1}=\{\vv_{i}\}$.  We showed above that the variables in $\Vnum{1}$ are uniquely determined by $\Fp=\bs{0}$.
Suppose that $\FpNum{1}$ involves some variables other than those in $\Vnum{1}$. Note that every polynomial in $\FpNum{1}$ involves at least one of the variables in $\Vnum{1}$. Let $\Vnum{2}$ be the set of all variables involved in $\FpNum{1}$. By \allFixedVarsLem, all the variables in $\Vnum{2}$ are uniquely determined by $\Fp=\bs{0}$.

We will now proceed inductively.  For any $\t \geq 2$, let $\Vnum{t}$ be a subset of \phantomsection\label{NnumDef}$\Nnum{t}$ variables in $\V$.  Assume that all the variables in $\Vnum{t}$ are uniquely determined by $\Fp=\bs{0}$.  Since $\dim \VV(\Fdp)=\dim \VV(\Fp)$, it follows that all the variables in $\Vnum{t}$ are finitely determined by $\Fdp=\bs{0}$.   It follows that $\Fdp$ must contain at least $\Nnum{t}$ algebraically independent polynomials, each involving at least one of the variables in $\Vnum{t}$.  Let $\FpNum{t}$ be this set of polynomials.  Suppose $\FpNum{t}$ involves some variables other than $\Vnum{t}$.  Define $\Vnum{t+1}$ to be the set of all variables involved in $\FpNum{t}$. By \allFixedVarsLem, all the variables in $\Vnum{t+1}$ are uniquely determined by $\Fp=\bs{0}$.

Since this is true for every $\t$, and there are finitely many variables, this process must terminate at some finite step \phantomsection\label{finalTDef}$\finalT$, at which point $\FpNum{T}$ is a set of $\Nnum{T}$ algebraically independent polynomials in $\Nnum{T}$ variables.

This means that all the variables in $\FpNum{T}$ are finitely determined by $\FpNum{T}=\bs{0}$, and since $\efi$ only involves a subset of the variables in $\FpNum{T}$, it follows that the polynomials in $\{\FpNum{T},\efi\} \subset \Fp$ are algebraically dependent.  Furthermore, since $\Fp$ is \mad\ by assumption, we have that $\FpNum{T}=\Fdp$.

Finally, observe that $\Fdp$ contains $\nOf(\Odp)$ polynomials in $\r(\mOf(\Odp)-\r)$ variables.  Since $\Fdp=\FpNum{T}$, and $\FpNum{T}$ has $\Nnum{T}$ polynomials in $\Nnum{T}$ variables, it follows that $\nOf(\Odp) = \r(\mOf(\Odp)-\r)$, as desired.
\end{proof}

\pagebreak
\begin{proof}(\independenceLem)
\begin{itemize}
\item[($\Rightarrow$)]
Suppose $\Fp$ is \mad.  By \basisLem, $\nOf(\Op)=\r(\mOf(\Op)-\r)+1 > \r(\mOf(\Op)-\r)$, and we have the first implication.

\item[($\Leftarrow$)]
Suppose there exists an $\Op$ with $\nOf(\Op)>\r(\mOf(\Op)-\r)$.  By \LLem, $\nOf(\Op)>\li(\Op)$, which implies the polynomials in $\Fp$, and hence $\Ftilde$, are algebraically dependent.
\end{itemize}
\end{proof}

\begin{proof}(\LRMCThm)
\begin{itemize}
\item[($\Rightarrow$)]
Suppose for contrapositive that for every $\Otilde$ formed with $\r(\d-\r)$ columns of $\O$, there exists an $\Op$ formed with a subset of its columns such that $\mOf(\Op) < \nOf(\Op)/\r+\r$.  \independenceLem\ implies that the polynomials in $\Fp$, and hence $\Ftilde$, are algebraically dependent.  It follows by \dimensionLem\ that there are infinitely many subspaces in $\SS(\XO)$.
\item[($\Leftarrow$)]
Suppose that for some $\Otilde$ formed with $\r(\d-\r)$ columns of $\O$, every $\Op$ formed with a subset of the columns in $\Otilde$ satisfies $\mOf(\Op) \geq \nOf(\Op)/\r+\r$, including $\Otilde$. By \independenceLem, the $\r(\d-\r)$ polynomials in $\Ftilde$ are algebraically independent.  It follows by \dimensionLem\ that there are at most finitely many subspaces in $\SS(\XO)$, hence at most finitely many rank-$\r$ completions of $\XO$.
\end{itemize}
\end{proof}

\section{Proof of \uniquenessThm}
\label{uniquenessSec}
In this section we give the proof of \uniquenessThm.  We will use \phantomsection\label{XtildeDef}$\Xtilde_{\fix{\Otilde}}$ and \phantomsection\label{XhatDef}$\Xhat_{\fix{\Ohat}}$ to denote the $\d \times \r(\d-\r)$ and $\d \times (\d-\r)$ submatrices of $\XO$ corresponding to $\Otilde$ and $\Ohat$.  In addition, let \phantomsection\label{oihatDef}$\oihat$ and \phantomsection\label{xhatiDef}$\xhat_{\fix{\oihat}}$ denote the $\i^{\rm th}$ columns of $\Ohat$ and $\Xhat_{\fix{\Ohat}}$.

In order to prove \uniquenessThm, we will require Theorem 1 in \cite{identifiability}, which we state here as the following lemma, with some minor adaptations to our context.

\begin{myLemma}
\label{identifiabilityLem}
Suppose $\Ohat$ is a $\d \times (\d-\r)$ matrix with binary entries for which \identifiabilityCond\ holds and let $\s \in \Gr(\r,\R^{\fix{\d}})$.
Then for $\nuuG$-almost every $\sstar$, $\{\s_{\fix{\oihat}} = S^\star_{\fix{\oihat}}\}_{\fix{\i=1}}^{\fix{\d-\r}}$ if and only if $\s = \sstar$.
\end{myLemma}

With this, we are ready to give the proof of \uniquenessThm.
\begin{proof}(\uniquenessThm)
Suppose $\O$ contains two disjoint matrices $\Otilde$ and $\Ohat$ satisfying the conditions of \uniquenessThm.

Since $\Otilde$ satisfies \finiteCond, by \LRMCThm\ there are at most finitely many $\r$-dimensional subspaces that \fit\ $\Xtilde_{\fix{\Otilde}}$.  Equivalently, the set $\Ftilde$, containing the $\r(\d-\r)$ polynomials defined by the columns in $\Xtilde_{\fix{\Otilde}}$, is algebraically independent.  Let \phantomsection\label{efihatDef}$\efihat$ be the polynomial defined by $\xoihat$.  It follows that the set $\{\Ftilde,\efihat\}$ is algebraically dependent.  Let $\Fdp$ be a subset of the polynomials in $\Ftilde$, such that $\Fp=\{\Fdp,\efihat\}$ is \mad.  Then any subspace $\s$ with basis $\U$ that \fits\ $\xoihat$ must satisfy $\Fp=\bs{0}$, implying by \oneFixedVarLem\ that $\U_{\fix{\oihat}}=\bs{{\rm U}}^\star_{\fix{\oihat}}$.

Therefore, every $\s$ that \fits\ both $\Xtilde_{\fix{\Otilde}}$ and $\Xhat_{\fix{\Ohat}}$ must satisfy $\{\s_{\fix{\oihat}} = \sstaroihat\}_{\fix{\i}=1}^{\fix{\d-\r}}$.  Since $\Ohat$ satisfies \identifiabilityCond, it follows by \identifiabilityLem\ that $\s=\sstar$.
\end{proof}

In \modelSec\ we mentioned that there are cases where $\r(\d-\r)$ columns with only $\r+1$ samples are sufficient for unique completability.  The next result states that this is indeed the case if $\r=1$.

\begin{myProposition}
\label{rankOneProp}
If $\r=1$, \finiteCompletability\ is equivalent to unique completability.
\end{myProposition}
\begin{proof}
Assume $\r=1$.  Then $\U_{\fix{\deltai}}$ and $\U_{\fix{\nabli}}$ are scalars, so $\efi$ simplifies into:
\begin{align*}
\efi \ = \ \left(\U_{\fix{\deltai}} \UstarNabli - \ \U_{\fix{\nabli}} \UstarDeltai \right) \tiStar.
\end{align*}
This implies that $\F=\bs{0}$ is a system of linear equations, hence if it has finitely many solutions, it has only one.
\end{proof}

In \modelSec\ we also mentioned that in general, strictly more than $\r(\d-\r)$ columns with only $\r+1$ samples are necessary for unique completability.    We would like to close this section with an example where this is the case.

\begin{myExample}
\label{necessaryNEg}
Consider $\d=4$ and $\r=2$, such that $\N=\r(\d-\r)=4$.  Let
\begin{align*}
\O \ = \ \left[ \begin{matrix}
1 & 1 & 1 & 0 \\
1 & 1 & 0 & 1 \\
1 & 0 & 1 & 1 \\
0 & 1 & 1 & 1
\end{matrix}\right].
\end{align*}
It is easy to see that that $\Otilde=\O$ satisfies the conditions of \LRMCThm.  One may also verify (for example, solving explicitly $\F(\V)=\bs{0}$) that for \ae\ $\X$ there exist two subspaces that \fit\ $\XO$.

As a matter of fact, this will also be the case for any permutation of the rows and columns of this matrix.  One may construct similar samplings with the same property for larger $\d$ and $\r$.  All this to say that this is not a singular pathological example; there are many samplings that cannot be uniquely recovered with only $\r(\d-\r)$ columns with $\r+1$ samples.
\end{myExample}

\section{Additional Proofs}
\label{additionalSec}
In this section we present the proofs of \LRMCLem\ and \probabilityThm.  The proof of \LRMCCor\ follows directly from \LRMCLem\ and \uniquenessThm, and the proof of \algorithmCor\ follows directly from Theorem 1 in \cite{identifiability}.

\begin{proof}(\LRMCLem)
Suppose $\Otilde$ contains disjoint matrices $\{\Ot\}_{\fix{\ttau=1}}^{\fix{\r}}$ satisfying the conditions of \LRMCLem.  Let $\Op$ be a matrix formed with a subset of the columns in $\Otilde$.  Then $\Op=[ \OpNum{1} \ \cdots \ \OpNum{r} ]$ for some matrices \phantomsection\label{OtpDef}$\{\Otp\}_{\fix{\ttau=1}}^{\fix{\r}}$ formed with subsets of the columns in $\{\Ot\}_{\fix{\ttau=1}}^{\fix{\r}}$.

It follows that
\begin{align*}
\nOf(\Op) \ = \ \sum_{\fix{\ttau=1}}^{\fix{\r}{}} \nOf(\Otp) \ \leq \ \sum_{\fix{\ttau=1}}^{\fix{\r}{}} \max_{\fix{\ttau}} \nOf(\Otp).
\end{align*}
Assume without loss of generality that this maximum is achieved when $\ttau=1$.  Then
\begin{align*}
\nOf(\Op) \ \leq \ \r \nOf(\OpNum{1}) \ \leq \ \r (\mOf(\OpNum{1}) - \r) \ \leq \ \r (\mOf(\Op) - \r),
\end{align*}
where the last two inequalities follow because \eqref{identifiabilityEq} holds for every $\Otp$ by assumption, and because $\mOf(\Op) \geq \mOf(\Otp)$ for every $\ttau$.

Since $\Op$ was arbitrary, we conclude that \eqref{LRMCEq} holds for every matrix $\Op$ formed with a subset of the columns in $\Otilde$.
\end{proof}

The following lemma shows that \identifiabilityCond\ is satisfied with high probability under uniform random sampling schemes with only $\Ord(\max\{ \r , \log\d\})$ samples per column.
\begin{myLemma}
\label{probabilityLem}
Let the assumptions of \probabilityThm\ hold, and let $\Ohat$ be a matrix formed with $\d-\r$ columns of $\O$.  With probability at least $1-\frac{\fix{\eps}{}}{\fix{\d}}$, $\Ohat$ will satisfy \identifiabilityCond.
\end{myLemma}

\begin{proof}
Let \phantomsection\label{EDef}$\E$ be the event that $\mOf(\Op)<\nOf(\Op) + \r$ for some matrix $\Op$ formed with a subset of the columns in $\Ohat$.  It is easy to see that this will only occur if there is a matrix $\Op$ formed with $\nOf$ columns of $\Ohat$ that has all its nonzero entries in the same $\nOf+\r-1$ rows.  Let \phantomsection\label{EnDef}$\En$ denote the event that the matrix formed with the first $\nOf$ columns from $\Ohat$ has all its nonzero entries in the first $\nOf+\r-1$ rows.  Then
\begin{align}
\label{badSetProbEq}
\P\left( \E \right) \ \leq \ \sum_{\fix{\nOf=1}}^{\fix{\d-\r}} {\d-\r \choose \nOf} {\d \choose \nOf+\r-1} \P\left( \En \right)
\end{align}
If each column of $\Ohat$ contains at least $\L$ nonzero entries, distributed uniformly and independently at random with $\L$ as in \eqref{LEq}, it is easy to see that $\P(\En)=0$ for $\nOf \leq \L-\r$, and for $\L-\r < \nOf \leq \d-\r$,
\begin{align*}
\P(\En) \ \leq \ \left(\frac{{\fix{\nOf}+\fix{\r}-1 \choose \fix{\L}{}}} {{\fix{\d} \choose \fix{\L}{}}}\right)^{\fix{\nOf}}
& \ < \ \left(\frac{\nOf+\r-1}{\d}\right)^{\fix{\L \nOf}}.
\end{align*}
Since ${\fix{\d}-\fix{\r} \choose \fix{\nOf}{}} < {\fix{\d}{} \choose \fix{\nOf}+\fix{\r}-1}$,
continuing with \eqref{badSetProbEq} we obtain:
\begin{align}
\P \left( \E \right) & \ < \ \sum_{\fix{\nOf=\L-\r+1}}^{\fix{\d-\r}} {\d \choose \nOf+\r-1}^2 \left(\frac{\nOf+\r-1}{\d}\right)^{\fix{\L \nOf}} \nonumber \\
& \ < \ \sum_{\fix{\nOf=\L}}^{\fix{\d/2}} {\d \choose \nOf}^2 \left(\frac{\nOf}{\d}\right)^{\fix{\L(\nOf-\r+1)}} \label{firstSumEq} \\
& \ + \ \sum_{\fix{\nOf=1}}^{\fix{\d/2}} {\d \choose \d-\nOf}^2 \left(\frac{\d-\nOf}{\d}\right)^{\fix{\L(\d-\nOf-\r+1)}} \label{secondSumEq}
\end{align}
For the terms in \eqref{firstSumEq}, write
\begin{align}
\label{firstTermEq}
{\d \choose \nOf}^2 \left(\frac{\nOf}{\d}\right)^{\fix{\L(\nOf-\r+1)}} & \ \leq \ \left({\frac{\d e}{\nOf}}\right)^{2\fix{\nOf}} \left(\frac{\nOf}{\d}\right)^{\fix{\L(\nOf-\r+1)}}.
\end{align}
Since $\nOf \geq \L \geq 2\r$,
\begin{align}
\label{someEq}
\eqref{firstTermEq} & \ < \ \left({\frac{\d e}{\nOf}}\right)^{2\fix{\nOf}} \left(\frac{\nOf}{\d}\right)^{\fix{\L} \frac{\fix{\nOf}}{2}} 
\ = \ e^{2\fix{\nOf}} \left(\frac{\nOf}{\d}\right)^{(\frac{\fix{\L}{}}{2}-2)\fix{\nOf}},
\end{align}
and since $\nOf \leq \frac{\fix{\d}{}}{2}$,
\begin{align}
\label{firstBoundEq}
\eqref{someEq}
& \ \leq \ e^{2\fix{\nOf}} \left(\frac{1}{2}\right)^{(\frac{\fix{\L}{}}{2}-2)\fix{\nOf}} 
\ = \ \left( e^2 \cdot 2^{-\frac{\fix{\L}{}}{2}+2} \right)^{\fix{\nOf}}
\ < \ \frac{\eps}{\d^2},
\end{align}
where the last step follows because $\L > 2\log_2(\frac{(\fix{\d} e)^2}{\fix{\eps}{}})+4$.

For the terms in\eqref{secondSumEq}, write
\begin{align}
\label{secondTermEq}
{\d \choose \d-\nOf}^2 \left(\frac{\d-\nOf}{\d}\right)^{\fix{\L(\d-\nOf-\r+1)}} & \ \leq \ \left(\frac{\d e}{\nOf}\right)^{2\fix{\nOf}} \left(\frac{\d-\nOf}{\d}\right)^{\fix{\L(\d-\nOf-\r+1)}}.
\end{align}
In this case, since $1 \leq \nOf \leq \frac{\fix{\d}{}}{2}$ and $\r \leq \frac{\fix{\d}{}}{6}$, we have
\begin{align*}
\eqref{secondTermEq}
\ < \ (\d e)^{2\fix{\nOf}} \left(\frac{\d-\nOf}{\d}\right)^{\fix{\L} \frac{\fix{\d}}{3}} 
& \ = \ (\d e)^{2\fix{\nOf}} \left[\left(1-\frac{\nOf}{\d}\right)^{\fix{\d}} \right]^{\frac{\fix{\L}{}}{3}} \\
& \ \leq \ (\d e)^{2\fix{\nOf}} \left[e^{-\fix{\nOf}}\right]^{\frac{\fix{\L}{}}{3}},
\end{align*}
which we may rewrite as
\begin{align}
\label{secondBoundEq}
\left(e^{2\log\fix{\d}}\right)^{\fix{\nOf}} \left(e^2\right)^{\fix{\nOf}} \left(e^{-\frac{\fix{\L}{}}{3}}\right)^{\fix{\nOf}} \ = \ \left(e^{2\log\fix{\d} + 2 - \frac{\fix{\L}{}}{3} }\right)^{\fix{\nOf}} \ < \ \frac{\eps}{\d^2},
\end{align}
where the last step follows because $\L > 3 \log(\frac{\fix{\d}^2}{\fix{\eps}{}}) + 6 \log \d + 6$.

Substituting \eqref{firstBoundEq} and \eqref{secondBoundEq} in \eqref{firstSumEq} and \eqref{secondSumEq}, we have that $\P(\E)<\frac{\fix{\eps}{}}{\fix{\d}}$.
\end{proof}

We are now ready to give the proof of \probabilityThm.

\begin{proof}(\probabilityThm)
If $\N \geq \r(\d-\r)$, randomly select disjoint matrices $\{\Ot\}_{\fix{\ttau=1}}^{\fix{\r}}$, each formed with $\d-\r$ columns of $\O$.

Union bounding over $\ttau$, we may upper bound the probability that $\O$ fails to satisfy the conditions of \LRMCLem\ by
\begin{align*}
\sum_{\fix{\ttau=1}}^{\fix{\r}{}} \P(\E) \ < \ \sum_{\fix{\ttau=1}}^{\fix{\r}{}} \frac{\eps}{\d} \ < \ \sum_{\fix{\ttau=1}}^{\fix{\r}{}} \frac{\eps}{\r} \ = \ \eps.
\end{align*}

The first part of the statement follows because the conditions in \LRMCLem\ imply the conditions in \LRMCThm.

If $\N \geq (\r+1)(\d-\r)$, randomly select disjoint matrices $\{\Ot\}_{\fix{\ttau=1}}^{\fix{\r+1}}$, each formed with $\d-\r$ columns of $\O$.  By the same arguments, the probability that $\O$ fails to satisfy the conditions of \uniquenessThm\ is upper bounded by:
\begin{align*}
\sum_{\fix{\ttau=1}}^{\fix{\r+1}} \P(\E) \ < \ \sum_{\fix{\ttau=1}}^{\fix{\r+1}} \frac{\eps}{\d} \ < \ \sum_{\fix{\ttau=1}}^{\fix{\r+1}} \frac{\eps}{\r+1} \ = \ \eps.
\end{align*}
\end{proof}

\section{Conclusions}
\label{conclusionsSec}
In this paper we give sampling conditions for \hyperref[finitelyCompletableDef]{finite rank-$\r$ completability}, that is, conditions on the set of observed entries to guarantee that a matrix can be completed in at most finitely many ways.  We also provide deterministic sampling conditions for unique completability that can be efficiently verified.  In addition, we show that uniform random samplings with only $\Ord(\max\{\r,\log\d\})$ observed entries per column satisfy these conditions with high probability.  These findings have several implications on \LRMC\ regarding lower bounds, sample and computational complexity, the role of coherence, adaptive settings and the validation of any completion algorithm.

\section{Acknowledgements}
\label{acknowledgementsSec}
We would like to thank Louis Theran for pointing out a mistake in a previous version of the paper. In that earlier version we erroneously assumed that columns with more than $\r+1$ observed entries would yield multiple independent constraints.  However, as Theran pointed out through the following example in \cite{kiraly}, these constraints may be algebraically dependent.  For this reason, in our current analysis we use only one constraint per column.

\begin{myExample}
\label{counterEg}
Suppose $\X$ is a rank-$2$ matrix observed on the entries indicated by $1$'s
\begin{align*}
\O \ = \ \left[\begin{matrix}
1 & 1 & 1 & 0 & 0 \\
1 & 1 & 1 & 0 & 0 \\
1 & 1 & 0 & 1 & 1 \\
0 & 0 & 1 & 1 & 1 \\
0 & 0 & 1 & 1 & 1
\end{matrix}\right].
\end{align*}
Here $\x_3$ is observed on $\r+2$ entries.  Using Definition 1 in our earlier version of this paper, $\x_3$ yields the two central columns of $\breve{\O}$
\begin{align*}
\breve{\O} \ = \ \left[\begin{matrix}
1 & 1 & 1 & 1 & 0 & 0 \\
1 & 1 & 1 & 1 & 0 & 0 \\
1 & 1 & 0 & 0 & 1 & 1 \\
0 & 0 & 1 & 0 & 1 & 1 \\
0 & 0 & 0 & 1 & 1 & 1
\end{matrix}\right].
\end{align*}
The two central columns in $\breve{\O}$ correspond to $\O_3$, and encode the two polynomial constraints obtained from the observed entries in $\x_3$.

While the sampling $\breve{\O}$ satisfies the completability conditions of Theorem 1 in our earlier version of this paper, $\X$ cannot be completed.  This is because the two polynomials defined by $\x_3$ share the same coefficients, which makes them algebraically dependent.  If instead of $\x_3$ we observed two columns on the entries indicated by $\O_3$, then we would also obtain two polynomial constraints.  Only these polynomials would have generic coefficients, and would be algebraically independent with probability $1$.  In fact, a matrix observed on the entries indicated in $\breve{\O}$ (or more) can indeed be completed.
\end{myExample}



\begin{thebibliography}{1}

\bibitem{collaborativeFiltering}
J.~Rennie and N.~Srebro, \emph{Fast maximum margin matrix factorization for collaborative prediction}, International Conference on Machine Learning, pp. 713-719, 2005.

\bibitem{weinberger}
K.~Weinberger and L.~Saul, \emph{Unsupervised learning of image manifolds by semidefinite programming}, International Journal of Computer Vision, 70(1), pp. 77-90, 2006.

\bibitem{candes-recht}
E.~Cand\`es and B.~Recht, \emph{Exact matrix completion via convex optimization}, Foundations of Computational Mathematics, 9(6), pp. 717-772, 2009.

\bibitem{candes-tao}
E.~Cand\`es and T.~Tao, \emph{The power of convex relaxation: near-optimal matrix completion}, IEEE Transactions on Information Theory, 56(5), pp. 2053-2080, 2010.

\bibitem{recht}
B.~Recht, \emph{A simpler approach to matrix completion}, Journal of Machine Learning Research, 12, pp. 3413-3430, 2011.

\bibitem{gross}
D.~Gross, \emph{Recovering low-rank matrices from few coefficients in any basis}, IEEE Transactions on Information Theory, 57(3), pp. 1548-1566, 2011.

\bibitem{chen}
Y.~Chen, \emph{Incoherence-optimal matrix completion}, IEEE Transactions on Information Theory, 61(5), pp. 2909-2923, 2013.

\bibitem{coherentLRMC}
Y.~Chen, S.~Bhojanapalli, S.~Sanghavi and R.~Ward, \emph{Coherent matrix completion}, International Conference on Machine Learning, pp. 674-682, 2014.

\bibitem{bhojanapalli}
S.~Bhojanapalli and P.~Jain, \emph{Universal matrix completion}, International Conference on Machine Learning, pp. 1881-1889, 2014.

\bibitem{rigidity}
A.~Singer and M.~Cucuringu, \emph{Uniqueness of low-rank matrix completion by rigidity theory}, SIAM Journal on Matrix Analysis and Applications, 31(4), pp. 1621-1641, 2010.

\bibitem{kiraly}
F.~Kir\'aly, L.~Theran and R.~Tomioka, \emph{The algebraic combinatorial approach for low-rank matrix completion}, Journal of Machine Learning Research, pp. 1391-1436, 2015.

\bibitem{iterative}
E.~Chunikhina, R.~Raich and T.~Nguyen, \emph{Performance analysis for matrix completion via iterative hard-thresholded SVD}, IEEE Statistical Signal Processing Workshop, pp. 392-395, 2014.

\bibitem{cai}
J.~Cai, E.~Cand\`es and and Z.~Shen, \emph{A singular value thresholding algorithm for matrix completion}, SIAM Journal on Optimization, 20(4), pp. 1956-1982, 2010.

\bibitem{keshavan10}
R.~Keshavan, A.~Montanari and S.~Oh, \emph{Matrix completion from a few entries}, IEEE Transactions on Information Theory, pp. 324-328, 2010.

\bibitem{fpc}
S.~Ma, D.~Goldfarb, L.~Chen, \emph{Fixed point and Bregman iterative methods for matrix rank minimization}, Mathematical Programming, 128(1-2), pp. 321-353, 2011.

\bibitem{grouse}
L.~Balzano, R.~Nowak, and B.~Recht, \emph{Online identification and tracking of subspaces from highly incomplete information}, Allerton Conference on Communication, Control and Computing, pp. 704-711, 2010.

\bibitem{jain}
P.~Jain, P.~Netrapalli and S.~Sanghavi, \emph{Low-rank matrix completion using alternating minimization}, ACM Symposium on Theory Of Computing, pp. 665-674, 2013.

\bibitem{ssp14}
D.~Pimentel-Alarc\'on, L.~Balzano and R.~Nowak, \emph{On the sample complexity of subspace clustering with missing data}, IEEE Statistical Signal Processing Workshop, pp. 280-283, 2014.

\bibitem{cox}
D.~Cox, J.~Little and D.~O'shea, \emph{Ideals, varieties, and algorithms: An introduction to computational algebraic geometry and commutative algebra}, Third Edition, Springer, 2007.

\bibitem{aramova}
A.~Aramova, L.~Avramov and J.~Herzog, \emph{Resolutions of monomial ideals and cohomology over exterior algebras}, Transactions of the American Mathematical Society, 352(2), pp. 579-594, 2000.

\bibitem{identifiability}
D.~Pimentel-Alarc\'on, N.~Boston and R.~Nowak, \emph{Deterministic conditions for subspace identifiability from incomplete sampling}, IEEE International Symposium on Information Theory, pp. 2191-2195, 2015.

\end{thebibliography}
\end{document}